\pdfoutput=1

\documentclass[11pt]{article}

\usepackage{emnlp2021}

\usepackage{times}
\usepackage{latexsym}

\usepackage[T1]{fontenc}

\usepackage[utf8]{inputenc}

\usepackage{microtype}

\usepackage{graphicx}
\graphicspath{ {./img/} }
\DeclareGraphicsExtensions{.pdf,.eps,.png} %
\usepackage{subcaption}
\usepackage{color}
\usepackage{amsmath}
\usepackage{amssymb}
\usepackage{amsthm}
\usepackage{thmtools} 
\usepackage{mathtools} 
\usepackage[notext]{stix}
\DeclareMathOperator{\nulls}{\varnothing}
\DeclareMathOperator{\myemptylist}{\ensuremath{[\,]}}
\DeclareMathOperator{\mysi}{\sharp}
\DeclareMathOperator{\myunk}{\texttt{<unk>}}

\theoremstyle{definition}

\newtheorem{lemma}{Lemma}

\newcommand{\mytxt}[1]{\texttt{#1}}
\newcommand{\mytxts}[1]{\mytxt{\#\##1}}
\newcommand{\mywhitespace}{\texttt{\char32}}
\newcommand{\mybm}[2]{\begin{bmatrix}#1 \\ #2\end{bmatrix}}

\theoremstyle{definition}
\newtheorem{Ddef}{Definition} %
\newenvironment{definition}    %
  {%
   \pushQED{\qed}\begin{Ddef}}
  {\popQED\end{Ddef}}
  
\theoremstyle{remark}
\newtheorem{XxmpX}{\normalfont \textbf{Example}} %
\newenvironment{example}    %
  {%
   \pushQED{\qed}\begin{XxmpX}}
  {\popQED\end{XxmpX}}
\usepackage{booktabs}
\usepackage{multirow}
\usepackage{dcolumn}
\usepackage{siunitx}
\usepackage{caption}

\usepackage[linesnumbered,ruled,vlined]{algorithm2e}
\newcommand\mykw[1]{\normalfont \textbf{#1}}

\SetCommentSty{mycommfont}
\SetKw{Continue}{continue}
\newcommand{\mysetind}{\SetInd{0.3em}{0.3em}}
\SetNlSkip{0.5em}
\let\oldnl\nl%
\newcommand{\nonl}{\renewcommand{\nl}{\let\nl\oldnl}}%

\newcommand{\eat}[1]{}

\newcommand{\lmmshort}{LinMaxMatch}
\newcommand{\etoeshort}{E2E~WordPiece}
\newcommand{\isatwordboundary}{IsWdBndry}

\usepackage{comment}

\usepackage{afterpage}

\usepackage{hyperref}

\title{Fast WordPiece Tokenization}
\newcommand{\authorsep}{\hspace{1em}}

\author{Xinying Song$^\dag$ \authorsep Alex Salcianu$^\dag$ \authorsep Yang Song$^\ddag$\Thanks{\hspace{0.1em} Research conducted while working at Google.} \authorsep Dave Dopson$^\dag$ \authorsep  Denny Zhou$^\dag$ 
\\ $^\dag$Google Research, Mountain View, CA \\ $^\dag$\texttt{\{xysong,salcianu,ddopson,dennyzhou\}@google.com} \\
$^\ddag$Kuaishou Technology, Beijing, China 
\\ $^\ddag$\texttt{yangsong@kuaishou.com}
}

\date{}

\begin{document}

\maketitle
\begin{abstract} 
Tokenization is a fundamental preprocessing step for almost all NLP tasks. In this paper, we propose efficient algorithms for the WordPiece tokenization used in BERT, from single-word tokenization to general text (e.g., sentence) tokenization. When tokenizing a single word, WordPiece uses a longest-match-first strategy, known as maximum matching. The best known algorithms so far are $O(n^2)$ (where $n$ is the input length) or $O(nm)$ (where $m$ is the maximum vocabulary token length). We propose a novel algorithm whose tokenization complexity is strictly $O(n)$.  Our method is inspired by the Aho-Corasick algorithm. We introduce additional linkages on top of the trie built from the vocabulary, allowing smart transitions when the trie matching cannot continue.  For general text, we further propose an algorithm that combines pre-tokenization (splitting the text into words) and our linear-time WordPiece method into a single pass. Experimental results show that our method is 8.2x faster than HuggingFace Tokenizers and 5.1x faster than TensorFlow~Text on average for general text tokenization.
\end{abstract}

\section{Introduction}

Tokenization is the process of splitting text into smaller units called tokens
(e.g., words).  It is a fundamental preprocessing step for almost all NLP
applications: sentiment analysis, question answering, machine translation, information retrieval, etc.

Modern NLP models like BERT~\cite{devlin-etal-2019-bert}, GPT-3~\cite{gpt-3},
and XLNet~\cite{yang-xlnet-2019} tokenize text into subword units~\cite{schuster-nakajima-japanese,sennrich-etal-2016-neural,kudo-2018-subword}.  As a
midpoint between words and characters, subword units retain linguistic meaning
(like morphemes), while alleviating out-of-vocabulary situations even with a
relatively small-size vocabulary.

In this paper, we propose efficient algorithms for WordPiece, the subword
tokenization used in BERT~\cite{devlin-etal-2019-bert}.  Given Unicode text that
has already been cleaned up and normalized, WordPiece has two steps: (1)
pre-tokenize the text into words (by splitting on punctuation and whitespaces),
and (2) tokenize each word into wordpieces.

For single-word tokenization, WordPiece uses a greedy longest-match-first strategy: iteratively pick the longest prefix of the remaining text that matches a vocabulary token. This is well-known as Maximum Matching or 
MaxMatch~\cite{handbook-nlp}, which 
has also been used for Chinese word segmentation since 1980s~\cite{liu-liang-1986}. 

Despite its wide use in NLP for decades, to the best of our knowledge, the most efficient MaxMatch algorithms so far are $O(n^2)$ (where $n$ is the input word length) or $O(nm)$ (where $m$ is the maximum vocabulary token length) (see Section~\ref{sec:related}).  It's worth noting that the latter has a vocabulary-specific multiplicative factor $m$, which can be large when the vocabulary contains long words.

We propose \lmmshort{}, a novel MaxMatch algorithm for WordPiece tokenization, whose tokenization time is strictly $O(n)$ without any vocabulary-specific multiplicative factors.
Inspired by the Aho-Corasick algorithm~\cite{aho-corasick-1975-efficient}, we organize vocabulary tokens in a trie~\cite{trie-first-reference} and introduce precomputed \textbf{failure links} and \textbf{failure pops}. During tokenization, if an input character does not match any trie edge, we perform smart transitions  
to avoid backtracking to earlier input characters.
This involves collecting the recognized tokens (i.e., failure pops) and moving to a trie node (via the failure link), from where we continue to match the same character (Section~\ref{sec:lmm}).  

For general text tokenization, referred to as end-to-end tokenization in this paper,
we propose \etoeshort{}, an end-to-end algorithm that combines pre-tokenization and WordPiece tokenization into a single, linear-time pass (Section~\ref{sec:e2ewp}). 

Experimental results show that our method is 8.2x faster than HuggingFace Tokenizers ~\cite{huggingface-whole-package} and 5.1x faster than TensorFlow~Text~\cite{tftext-whole-package} on average for general text tokenization (Section~\ref{sec:experiments}). 

Although tokenization is relatively faster than other steps, it's still worth improving the performance: Tokenization is a prerequisite step for almost all NLP tasks, and any improvement on its efficiency helps reduce the latency of the entire inference. 
One potential impact of the work, for example, is on mobile NLP applications. On-device models are generally highly optimized for reducing latency, e.g., by distilling or compressing larger models. Thus, the impact of tokenization can be significant here.
Another impact is on aggregate computational savings for Web services like Google, Facebook, Twitter, etc. For example, Google uses BERT to power its Web search nowadays.\footnote{\url{https://blog.google/products/search/search-language-understanding-bert/}} Google serves billions of search queries per day, and it processes hundreds of trillions of Web pages in index building. By employing a faster tokenization system, the aggregate computational savings would be material, which also benefits the environment (for less power consumption). 

This paper also makes a theoretical contribution. The proposed \lmmshort{} algorithm solves the decades-old MaxMatch problem in the optimal $O(n)$ time, and the idea is applicable to other string matching or rewriting problems (Section~\ref{sec:algo-discussions}).

The code will be available at \url{https://www.tensorflow.org/text}.

\section{Related Work}
\label{sec:related}

Maximum Matching (or MaxMatch) has been used for Chinese word segmentation (CWS) since the 1980s~\cite{liu-liang-1986, %
handbook-nlp}. Recent CWS work focuses on machine learning-based segmentation approaches, but MaxMatch remains a commonly referenced baseline~%
\cite{chang-etal-2008-optimizing}.

More recently, subword tokenization techniques have become a near-universal feature of modern NLP models, including BERT~\cite{devlin-etal-2019-bert}, 
GPT-3~\cite{gpt-3}, 
XLNet~\cite{yang-xlnet-2019}, etc.
Common subword tokenization techniques include Byte-Pair Encoding (BPE)~\cite{schuster-nakajima-japanese, sennrich-etal-2016-neural}, SentencePiece~\cite{kudo-2018-subword} (based on unigram language modeling), and WordPiece~\cite{bertoss}.

The widely-adopted MaxMatch algorithm, which is used in the original WordPiece algorithm~\cite{bertoss}, starts from the longest possible prefix and decrements the length in search of the longest-matching token~\cite{jie-liu-liang-1989}.  
A variant starts from the shortest substring and increases the length~\cite{webster-kit-1992-tokenization, reps-1998,sassano-2014-deterministic}.
The worst-case time complexity of the previous algorithms are $O(n^2)$ or $O(nm)$ or even higher than that.\footnote{The exact complexity depends on implementation details, e.g., whether substring hashes are computed from scratch or incrementally, how substrings are searched in vocabulary, etc. 
}\footnote{Previous studies usually do not explicitly state the vocabulary-related multiplicative factor in the complexity, or just treat it as a hidden constant.} 
For example, the complexity of \citet{sassano-2014-deterministic} is $O(nm)$ (in our notations), since \texttt{Lookup(t,c,i,N)} (Figure 1 in their paper) may take $O(m)$ time (which is similar to the analysis in Section~\ref{sec:idea} of this paper).
\citet{reps-1998} recognizes maximum matching tokens using regular expressions in the context of compilers; their complexity is $O(|Q|n)$, where $|Q|$ is the number of states in the automaton built from the grammar/vocabulary. If applied to WordPiece tokenization, since vocabulary tokens are finite strings, their complexity can be refined as $O(nm)$.

Our algorithm is inspired by the Aho-Corasick algorithm~\cite{aho-corasick-1975-efficient}, but the two algorithms are designed to address different problems. 
Aho-Corasick is not optimal for the MaxMatch problem. %
In the worst-case scenario where every substring in the input matches a vocabulary token, Aho-Corasick finds a quadratic number of matches, resulting in an overall quadratic complexity for MaxMatch.
By comparison, our algorithm achieves the worst-case linear complexity for MaxMatch due to a novel definition of failure links, the newly-introduced failure pops, as well as a different way of emitting tokens.

It's worth clarifying the difference between our failure links and the tabular solution of \citet{reps-1998}. In their work, a table called \texttt{failed\_previously} is used to store whether a state \texttt{<q,i>} has been seen before in a failed attempt to match a token (where \texttt{q} is a state of the automaton and \texttt{i} is a position of the input). \citet{reps-1998} uses that table to avoid wasteful revisits of the same state. The table entries \texttt{<q,i>} depend on both the grammar/vocabulary and the actual input. In contrast, our failure links capture which state to transit to when trie matching cannot continue (Definition~\ref{def:failure-main}), and they are precomputed based on the vocabulary only, independent of the input. 

Finally, we discuss the complexity of algorithms for Byte-Pair Encoding (BPE)~\citep{schuster-nakajima-japanese, sennrich-etal-2016-neural} and SentencePiece~\citep{kudo-2018-subword}.
Note that they are different problems from MaxMatch (the topic of this paper). SentencePiece is based on unigram language modeling, and the optimal segmentation can be found in $O(nm)$ time with the Viterbi algorithm~\citep{viterbi-1967}.
BPE algorithms can be implemented in two ways. One is to enumerate the symbol pairs in the order that they were added to the vocabulary in the building phase. For each symbol pair, we scan the current sequence and replace all their occurrences with the merged symbol. The complexity is $O(|V|n)$, where $|V|$ is the size of the vocabulary. 
The other approach is to repeatedly select the pair of symbols from the current sequence that has the highest priority (e.g., the maximum frequency). Using a heap, this approach can be done in $O(nlogn)$.

\section{Linear-Time Single-Word Tokenization} %
\label{sec:lmm}
In this section, we present \lmmshort{}, an $O(n)$ algorithm
for single-word WordPiece tokenization.

\subsection{Background and Notations} %
\label{sec:background}

Given a vocabulary,\footnote{The construction of the 
vocabulary is outside the scope of this paper.  We refer the interested reader
to~\citet{tftext-whole-package}.} WordPiece tokenizes a word using the MaxMatch
approach: iteratively pick the longest prefix of the remaining text that
matches a vocabulary token until the entire word is segmented. 
If a word cannot be tokenized, the entire word is mapped to a special
    token $\myunk$.

WordPiece tokenization distinguishes wordpieces at the start of a word
    from wordpieces starting in the middle.  The latter start with a special symbol
    \mytxts{} (in BERT), which is called the \textbf{suffix indicator} and is denoted as
    $\mysi$ in this paper. Our method works with any suffix indicator: \mytxts{}, an arbitrary string, or the empty string (i.e., no distinction between the two kinds of wordpieces).

For example, the word \mytxt{johanson} may be tokenized as [\mytxt{johan},
\mytxts{son}].

We use the running example from Figure~\ref{fig:example}.
Table~\ref{tab:notation} summarizes our notations.
We construct a trie from the vocabulary $V$.  We use $\delta(u,c)=v$ to denote a
trie edge from node $u$ to node $v$ with character $c$ as the label.  If there
is no outgoing edge from $u$ with label $c$, $\delta(u,c)=\nulls$.
Let $\chi_v$ be the string represented by the node $v$, that is, the string
obtained by concatenating all the edge labels along the path from the root to
node $v$.
Let $r$ be the root of the trie and $r_{\mysi}$ be the node for the suffix
indicator $\mysi$. Obviously, $\chi_r = \varepsilon$ (where $\varepsilon$
denotes the empty string) and $\chi_{r_{\mysi}} = \mysi$.
The depth of node $v$ is defined as the number of characters in $\chi_v$ excluding the suffix indicator prefix (if any). Hence, the depth of $r$ or $r_{\mysi}$ is 0. In Figure~\ref{fig:graph3}, nodes 0 and 2 have depth 0, nodes 1, 3, and 8 have depth 1, node 10 has depth 2, etc.

\begin{table}[htb!]
    \centering
\newcommand\myoptional[1]{#1}
\renewcommand\myoptional[1]{}   %
\setlength\tabcolsep{2pt} %
\begin{tabular}{ll}
\textbf{Symbol} \myoptional{& \textbf{Section}} & \textbf{Meaning} \\
\hline
    $\varepsilon$ \myoptional{& \ref{sec:problem}} &  The empty string \\
    $\mysi$ \myoptional{& \ref{sec:problem}} & The suffix indicator string \\
    $V$ \myoptional{& \ref{sec:problem}} & The vocabulary \\
    $\myunk$ \myoptional{& \ref{sec:problem}} & The unkown token \\
    $w,s$ \myoptional{& \ref{sec:problem}} & A string \\
    $c$ \myoptional{& \ref{sec:concepts}} & A character \\
    $\mywhitespace$ \myoptional{& \ref{sec:concepts}} & A whitespace character \\ 
    $r, r_{\mysi}$ \myoptional{& \ref{sec:concepts}} & The trie root and the node for $\mysi$ \\
    $u, v$ \myoptional{& \ref{sec:concepts}} & Trie nodes; $u$ is often the parent of $v$ \\
    $\nulls$ \myoptional{& \ref{sec:concepts}} & Null node \\
    $\delta(u,c)$ \myoptional{& \ref{sec:concepts}} & Trie edge from node $u$, with label $c$ \\
    $\chi_v$ \myoptional{& \ref{sec:concepts}} & The string represented by node $v$ \\
    $f(v), F(v)$ \myoptional{& \ref{sec:concepts}} & Failure link and failure pops\\ 
    $n$ \myoptional{& \ref{sec:problem}} & The length of the input \\ 
    $m$ \myoptional{& \ref{sec:problem}} & The maximum length of tokens in $V$ \\ 
    $M$ \myoptional{& \ref{sec:problem}} & The sum of the lengths of tokens in $V$ \\ 
\end{tabular}
\caption{Notations.}
\label{tab:notation}
\end{table}
\hspace{-10em}

\subsection{Intuition}
\label{sec:idea}

\begin{figure*}[htb!]
\begin{center}
\begin{subfigure}[b]{0.48\textwidth}
\setlength\tabcolsep{6pt} %
\begin{tabular}{|c|}
\hline
$\pmb{V}$: $\left\{ \mytxt{a}, \, \mytxt{abcdx},  \, \mytxts{b},  \, \mytxts{c}, \, \mytxts{cdy}, \,  \mytxts{dz} \right\}$ \\
\hline
\end{tabular}
\includegraphics[width=1.0\linewidth]{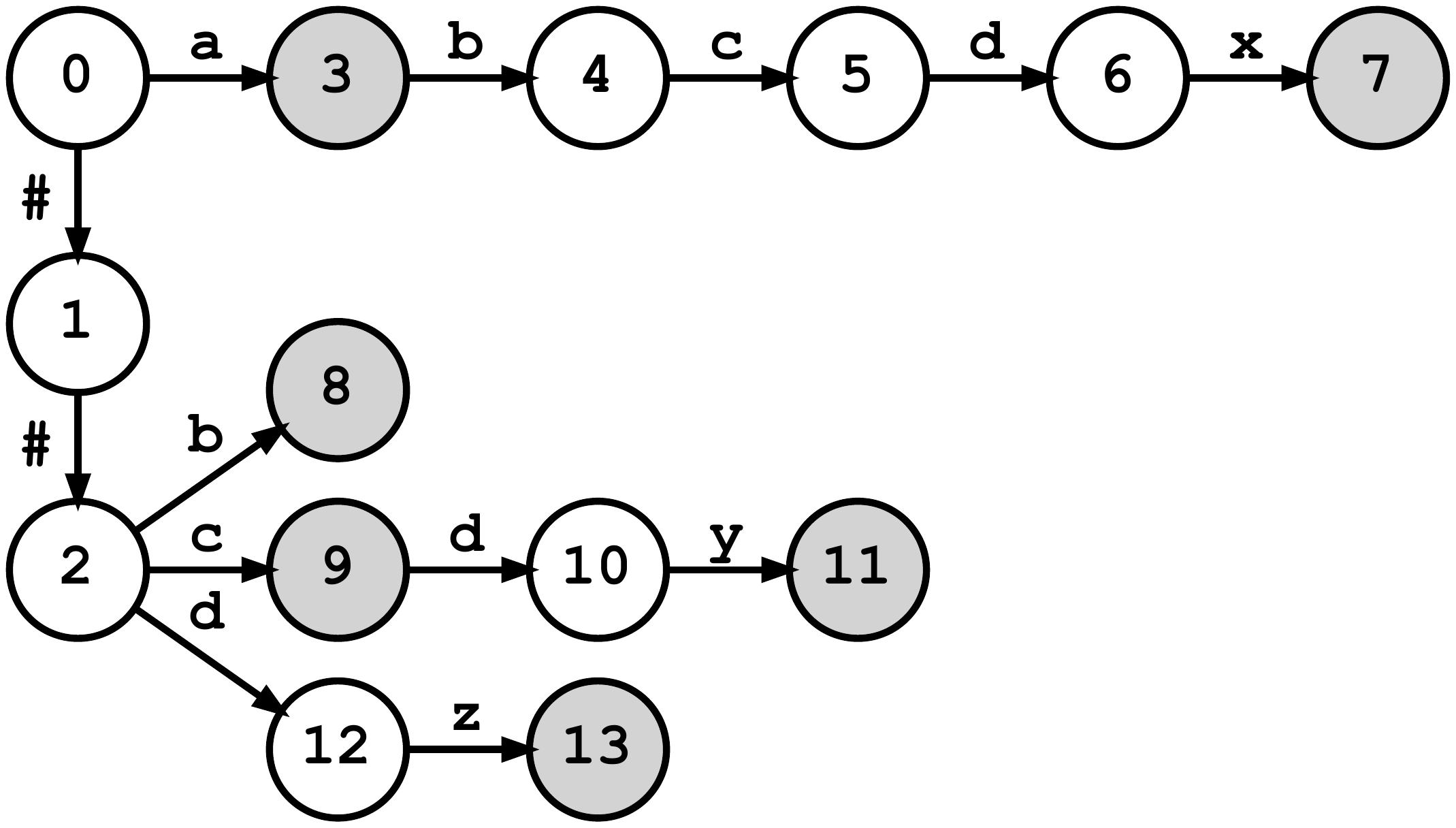}\llap{\includegraphics[width=0.4\linewidth]{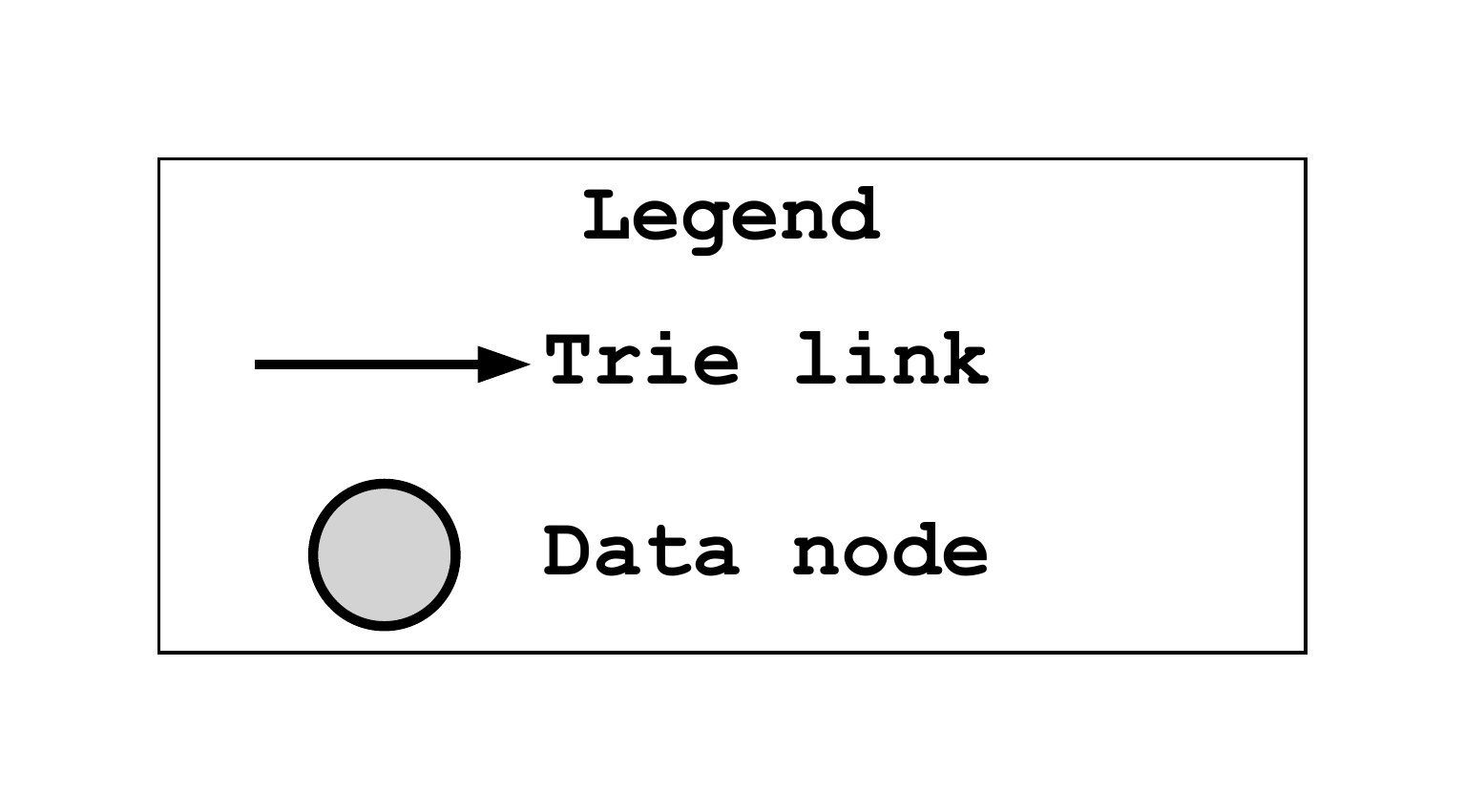}}
\caption{The vocabulary and the corresponding trie.}
\label{fig:graph3}
\end{subfigure}
\quad
\begin{subfigure}[b]{0.48\textwidth}
\setlength\tabcolsep{1pt} %
\begin{tabular}{c|ccccccccccc}
    $\pmb{v}$&  \textbf{0} & \textbf{1} & \multicolumn{2}{c}{\textbf{2}} \\
    \hline
    $\pmb{F(v)}$& $\myemptylist$ & $\myemptylist$ & \multicolumn{2}{c}{$\myemptylist$} \\ %
    $\pmb{f(v)}$& $\nulls$ & $\nulls$ & \multicolumn{2}{c}{$\nulls$} \\ %
    \hline
    \hline
    $\pmb{v}$&  \textbf{3} & \textbf{4} & \multicolumn{3}{c}{\textbf{5}} & \multicolumn{3}{c}{\textbf{6}}  & \multicolumn{3}{c}{\textbf{7}} \\
    \hline
    $\pmb{F(v)}$& [\mytxt{a}] & [\mytxt{a}] & \multicolumn{3}{c}{[\mytxt{a}, \mytxts{b}]} & \multicolumn{3}{c}{[\mytxt{a}, \mytxts{b}]} & \multicolumn{3}{c}{[\mytxt{abcdx}]} \\
    $\pmb{f(v)}$& 2 & 8 & \multicolumn{3}{c}{9} & \multicolumn{3}{c}{10} & \multicolumn{3}{c}{2} \\
    \hline
    \hline
    $\pmb{v}$ & \textbf{8} & \textbf{9} &  \multicolumn{2}{c}{\textbf{10}} & \multicolumn{3}{c}{\textbf{11}} & \multicolumn{2}{c}{\textbf{12}} & \multicolumn{2}{c}{\textbf{13}}  \\
    \hline
    $\pmb{F(v)}$ & [\mytxts{b}] & [\mytxts{c}] & \multicolumn{2}{c}{[\mytxts{c}]} & \multicolumn{3}{c}{[\mytxts{cdy}]} & \multicolumn{2}{c}{$\myemptylist$} & \multicolumn{2}{c}{[\mytxts{dz}]}  \\
    $\pmb{f(v)}$ & 2 & 2 & \multicolumn{2}{c}{12} & \multicolumn{3}{c}{2} & \multicolumn{2}{c}{$\nulls$} & \multicolumn{2}{c}{2}\\
\end{tabular}
\caption{Complete table of $f(v)$ and $F(v)$. }
\label{fig:F}
\end{subfigure}
\end{center}
\caption{Example vocabulary, the corresponding trie, and the table of auxiliary links and data. The suffix indicator is \mytxts{}. %
Node 0 is the root node. Data nodes (in grey) indicate vocabulary tokens, i.e., the represented string is in $V$.}
\label{fig:example}
\end{figure*}

To motivate our linear algorithm, let's first consider an alternative approach to MaxMatch using a simple vocabulary trie: 
when searching the longest token at a position, it starts from the shortest substring and iterates over the input text from left to right, following trie matching to find the longest prefixes that matches a vocabulary token.  %

\begin{example}
\label{example:motivation}
Consider the vocabulary and the trie from Figure~\ref{fig:graph3}, with the input string \mytxt{abcdz}. The expected output is $[\mytxt{a}, \mytxts{b}, \mytxts{c}, \mytxts{dz}]$.

Starting from position 0, we follow the trie edges to match the input characters from \mytxt{a} to \mytxt{d}, arriving at node 6. No trie edge exits node 6 with character \mytxt{z} as the label. The longest matching prefix seen so far 
is \mytxt{a}, which is the first recognized token.  
\end{example}

The challenge of this approach is that, when the trie fails to match the next character, the longest vocabulary token match may be several characters back. As shown in Example~\ref{example:motivation}, from position 0 we've matched the prefix \mytxt{abcd} 
but found that the longest matching token is \mytxt{a}.  When looking for the next token, we reset the start position at character \mytxt{b} and reprocess \mytxt{bcd..}, resulting in repetitive and wasteful iterations. The time complexity is $O(nm)$.

The idea of \lmmshort{} is to use precomputed information to avoid reprocessing the characters. %

\begin{example}
\label{example:motivation2}
    For the same example as above, when the trie matching fails at 
    character \mytxt{z}, since \mytxt{abcd} has been matched, given the vocabulary in use (Figure~\ref{fig:graph3}), we should be able to know that the first two
    longest-matching tokens are $[\mytxt{a}, \mytxts{b}]$. After collecting the tokens, we should reset our state as if we just matched \mytxts{cd} and then continue to match the same character \mytxt{z}. No need to reprocess \mytxt{bcd}.
\end{example}

Specifically, when trie matching arrives at node $v$ but cannot continue further, it must have matched the string represented by $v$ (i.e.\ $\chi_v$). We consider the tokens that MaxMatch would generate for the beginning of $\chi_v$  (called ``failure pops'' $F(v)$), which should be popped off the beginning of $\chi_v$ and put into the result. After that, we should transit to a state (following the ``failure link'' $f(v)$) that corresponds to the remaining suffix of $\chi_v$, from which the algorithm continues to match the next character.  $F(v)$ and $f(v)$ are defined as below and can be precomputed based on the vocabulary.

\begin{definition}{\textbf{Failure links and pops}}.
\label{def:failure-main}
Given a node $v$ and the corresponding string $\chi_v$, consider the shortest non-empty list of longest-matching-prefix tokens $[p_1,p_2,...,p_k]$ (where $p_i \in V$, $p_i \neq \varepsilon$ or $\mysi$,  for $1 \leq i \leq k$) that we can remove from $\chi_v$ (in order) until the remaining suffix can be represented by some node $v'$ from the trie. 

We define \textbf{failure pops} for node $v$ as $F(v)=[p_1,p_2,...,p_k]$ and \textbf{failure link} as $f(v)=v'$.

If such a non-empty list $[p_1,p_2,...,p_k]$ does not exist, we define $f(v)=\nulls$.  
$F(v)$ is undefined and unused in this case.
\end{definition}

Put it another way, $F(v)$ and $f(v)$ are defined by finding the longest prefix of the string $\chi_v$ that matches a vocabulary token, popping it, and repeating this procedure until the suffix string is found on the trie. 
Figure~\ref{fig:F} shows $F(v)$ and $f(v)$ computed for the example vocabulary and trie.

For readers with the background of finite-state transducers (FSTs) \cite{mohri-1997-finite}, it's helpful to see that $f(v)$ is related to the state transition function and $F(v)$ is related to the output function (more discussions in Section~\ref{sec:algo-discussions}).

\subsection{\lmmshort{} Tokenization}
\label{sec:inference}

Assume that, based on the vocabulary, we have precomputed the trie, failure links, and failure pops (precomputation is discussed in Section~\ref{sec:initialization}). Given an input string, we follow the trie edges to process the input characters one by one. When trie matching cannot continue from node $v$, we make a \textbf{failure transition} in two steps: (1) retrieve failure pops $F(v)$ and append to the end of tokenization result, and (2) follow the failure link to node $f(v)$.
After that, we continue from the new node $f(v)$. 

Algorithm~\ref{alg:inference} shows the tokenization algorithm. For now, ignore lines~\ref{alg:inference:check_edge_case}-\ref{alg:inference:edge_case}; we explain it later.

\begin{algorithm}[htb!]
\DontPrintSemicolon
\mysetind
\SetKwProg{Fn}{Function}{:}{}
\nonl
\Fn{\textsc{\lmmshort{}}{\normalfont (}$w${\normalfont )}}{
  \text{tokens}, $u$, $i$ $\leftarrow$ \textsc{MatchLoop}$(w\mywhitespace{}, 0)$ \label{alg:inference:call} \;
  \If{$i<|w|$ \mykw{or} $u \notin \{r, r_{\mysi}\}$ \label{alg:inference:check_returns}}{tokens $\leftarrow [\myunk]$ \label{alg:inference:reset}}
    \ElseIf{
    $
    u = r_{\mysi}  
      \mykw{ and } |\text{tokens}|=0 $ \label{alg:inference:check_edge_case}}{
        tokens $\leftarrow \textsc{OriginalWordPiece}(\mysi)$ \label{alg:inference:edge_case}
    }
  \Return tokens \label{alg:inference:return}
}
\nonl
\Fn{\textsc{MatchLoop}{\normalfont (}$s$, $i${\normalfont )}}{
  $u, \text{tokens}$ $\;\leftarrow\;$ $r, \myemptylist$ \;
  \While{$i < |s|$ \label{alg:inference:step_start}}{
    \While{$\delta(u,s[i]) = \nulls$ \label{alg:inference:failurewhilebegin}} {
        \lIf{$f(u) = \nulls$} { \Return tokens, $u$, $i$ \label{alg:inference:matchloopbreak}}
        tokens $\leftarrow$ \textsc{Extend}(tokens, $F(u)$) \label{alg:inference:append} \;
        $u \leftarrow$ $f(u)$ \label{alg:inference:fv}\;
    } 
    $u \leftarrow$ $\delta(u, s[i])$ \label{alg:inference:goto} \;
    $i \leftarrow i+1$ \label{alg:inference:iplus1}
  }
  \Return tokens, $u$, $i$  \label{alg:inference:final_return}
}
\caption{\lmmshort{} Tokenization}
\label{alg:inference}
\end{algorithm}

The main function calls \textsc{MatchLoop}() with two inputs: $w$ appended by a whitespace \mywhitespace{} and the start position 0 (line~\ref{alg:inference:call}). 
Inside that function, let's use the term \textbf{step} to denote an iteration of the loop on lines~\ref{alg:inference:step_start}-\ref{alg:inference:iplus1}, which processes one input character $s[i]$.
Each step starts from the current node $u$ and follows 
$f(u)$ zero, one, or multiple times (line~\ref{alg:inference:fv}), appending the tokens in $F(u)$ to the result along the way (line~\ref{alg:inference:append}), until it finds a trie edge that matches the current character (line~\ref{alg:inference:failurewhilebegin}) or $f(u)=\nulls$ (line~\ref{alg:inference:matchloopbreak}).

If the input string $w$ can be tokenized, the loop continues until $i\!=\!|s|\!-\!1$ pointing to the final appended whitespace. We know that $\delta(u, \mywhitespace{})=\nulls$ for any $u$ (since whitespace is not in any vocabulary token). \textsc{MatchLoop}() will keep following $f(u)$ while collecting $F(u)$ tokens along the way (line~\ref{alg:inference:append}-\ref{alg:inference:fv}) until it arrives at $r_{\mysi}$, where $f(r_{\mysi})=\nulls$. \textsc{MatchLoop}() returns on line~\ref{alg:inference:matchloopbreak} with $u =r_{\mysi}$, $i\!=\!|s|\!-\!1\!=\!|w|$, and tokens being the expected result (see Example~\ref{example:inference}). If $w=\varepsilon$, \textsc{MatchLoop}() returns immediately with $u=r$, $i\!=\!0\!=\!|w|$, and empty tokens. In either case, the tokens are returned by the main function (line~\ref{alg:inference:return}). 

On the other hand, if the word cannot be tokenized, when \textsc{MatchLoop}() returns on line~\ref{alg:inference:matchloopbreak}, there are two cases: (1) Some normal input character cannot be consumed after attempting failure transitions (i.e., $i\!<\!|w|$). (2) $i\!=\!|w|$ but the final landing node $u \notin \{r, r_{\mysi}\}$ representing a non-empty string $\chi_u$ yet $f(u)=\nulls$; according to Definition~\ref{def:failure-main}, $\chi_u$ cannot be tokenized. In either case, the result tokens are reset to $[\myunk]$ (line~\ref{alg:inference:reset}). See Example~\ref{example:inference2}.

Line~\ref{alg:inference:final_return} is only for safety reasons; it will not be visited since a whitespace is appended at the end. %

\begin{example}
\label{example:inference}
Consider $s=w\mywhitespace{}=\mytxt{abcdz\mywhitespace{}}$, using the vocabulary from Figure~\ref{fig:graph3}. The expected tokenization is $[\mytxt{a}, \mytxts{b}, \mytxts{c}, \mytxts{dz}]$.

\begin{table}[htb]
    \centering
\setlength\tabcolsep{2.2pt} %
\begin{tabular}{rrrlr}
  step & $i, s[i]$ & \multicolumn{2}{r}{node transition} &      result tokens\\
    \hline
       &      &          &   $\;\;\;$  0 &      $\myemptylist$ \\
     1 &   0, \mytxt{a}  & $\delta(0,\mytxt{a})$ & $\rightarrow$  3 &      $\myemptylist$ \\
     2 &   1, \mytxt{b}  & $\delta(3, \mytxt{b})$ & $\rightarrow$    4 &      $\myemptylist$ \\
     3 &   2, \mytxt{c}  & $\delta(4, \mytxt{c})$ & $\rightarrow$    5 &       $\myemptylist$ \\
     4 &   3, \mytxt{d}  & $\delta(5, \mytxt{d})$ & $\rightarrow$    6 &       $\myemptylist$ \\
     5 &   4, \mytxt{z}  & $f(6)$        &  $\rightarrow$  10 &         [\mytxt{a},\mytxts{b}] \\
       &               & $f(10)$        & $\rightarrow$   12 &         [\mytxt{a},\mytxts{b},\mytxts{c}] \\
       &               & $\delta(12, \mytxt{z})$ & $\rightarrow$   13 &         [\mytxt{a},\mytxts{b},\mytxts{c}] \\
     6 &   5, \mywhitespace{} & $f(13)$        &  $\rightarrow$   2 &      [\mytxt{a},\mytxts{b},\mytxts{c}, \mytxts{dz}] \\
       &              & $f(2)=\nulls$ & &      [\mytxt{a},\mytxts{b},\mytxts{c}, \mytxts{dz}] \\
\end{tabular}
\caption{Sequence of node transitions and result tokens.}
\label{tab:inference-example}
\end{table}

Table~\ref{tab:inference-example} shows the sequence of node transitions and result tokens in \textsc{MathLoop}().  The first row is the original state.  Steps 1-4 are self-explanatory.

Step 5 is more complex: when we reach step 5, the prefix \mytxt{abcd} has already been processed. The current node is node 6, and the next character is \mytxt{z}. As $\delta(6, \mytxt{z})=\nulls$, we copy $F(6)$ to the result (which becomes [\mytxt{a}, \mytxts{b}]) and follow $f(6)$ to node 10. Next, as $\delta(10, \mytxt{z})=\nulls$, we copy $F(10)$ to the result (which becomes
[\mytxt{a}, \mytxts{b}, \mytxts{c}]) and follow $f(10)$ 
to node 12.  Now, as $\delta(12, \mytxt{z})=13$, we follow the trie edge to node 13 and proceed to step 6.

Step 6 processes \mywhitespace{}. We first follow $f(13)$ to node 2, appending \mytxts{dz} to the result tokens. Then, at node 2 (i.e., $u=2=r_{\mysi}$), $\delta(u, \mywhitespace{})=\nulls$ and $f(u)=\nulls$. \textsc{MatchLoop}() returns on line~\ref{alg:inference:matchloopbreak}.

Back to the main function (line~\ref{alg:inference:check_returns}), since $i\!=\!5\!=\!|w|$ (meaning that \textsc{MatchLoop}() stopped at the final whitespace) and $u=r_{\mysi}$ (meaning that all matched characters \mytxt{abcd} are covered by the result tokens), the word is successfully tokenized. It returns [\mytxt{a}, \mytxts{b}, \mytxts{c}, \mytxts{dz}] as expected.
\end{example}

\begin{example}
\label{example:inference2}
Consider two input words $s_1\!=\!w_1\mywhitespace{}\!=\!\mytxt{abcz}\mywhitespace{}$, $s_2\!=\!w_2\mywhitespace{}\!=\!\mytxt{abcd}\mywhitespace{}$. Using the same vocabulary, neither $w_1$ nor $w_2$ can be tokenized.

For $s_1$, \textsc{MatchLoop}() consumes \mytxt{abc} but not \mytxt{z}. Hence it stops within the word: $i=3<|w_1|$. 

For $s_2$, \textsc{MatchLoop}() consumes all normal characters \mytxt{abcd} but not the whitespace \mywhitespace{}. When it returns on line~\ref{alg:inference:matchloopbreak}, $i=|w_2|$, $u$ is node 12 (since $f(12)\!=\!\nulls$), and the result tokens are $[\mytxt{a},\mytxts{b},\mytxts{c}]$, which do not cover character \mytxt{d}. Actually, the string \mytxts{d} represented by node 12 cannot be tokenized. 

Tokens are reset to $[\myunk]$ in both cases.
\end{example}

\paragraph{Corner cases}
One behavior of the original WordPiece algorithm~\cite{bertoss} is that, if the input starts with the suffix indicator, the first result token may start with the suffix indicator. For example, in Figure~\ref{fig:example}, if the input is \mytxts{bc}, the tokenization result is [\mytxts{b}, \mytxts{c}]. In this paper, by having $r_{\mysi}$ as a descendant of $r$, \lmmshort{} follows the same behavior and returns the same result. 

Because $r_{\mysi}$ is set as a descendant of $r$, if the input $w$ is 
$\mysi$ itself (e.g., \mytxts{}), normally Algorithm~\ref{alg:inference} would have returned an empty list of tokens, which is inconsistent with \citet{bertoss}.
We handle this as a special case.
Line~\ref{alg:inference:check_edge_case} checks whether $w$ is $\mysi$ by the following (instead of directly comparing the strings): if and only if $w=\mysi$, 
the landing node $u$ is $r_{\mysi}$ and the result tokens are empty after  
consuming all normal input characters (i.e., $i=|w|$)\footnote{Note that $i=|w|$ is satisfied implicitly on line~\ref{alg:inference:check_edge_case} 
(Algorithm~\ref{alg:inference}) since it's an \emph{else} statement following the \emph{if} statement on line~\ref{alg:inference:check_returns}.}.
If so, the tokens are reset by the precomputed result of the original WordPiece algorithm on $\mysi$  (line~\ref{alg:inference:edge_case}).  %

Algorithm~\ref{alg:inference} can be proved to
be consistent with
the original WordPiece algorithm~\cite{bertoss}.

\subsection{\lmmshort{} Precomputation}
\label{sec:initialization}

Given a vocabulary, it is straightforward to build the trie. This section explains how to precompute failure links $f(\cdot)$ and failure pops $F(\cdot)$.

We could compute $f(\cdot)$ and $F(\cdot)$ by directly using the procedure from Definition~\ref{def:failure-main}. 
Instead, we propose a faster algorithm (see Section~\ref{sec:complexity} for complexity).  Our algorithm computes $f(v)$ and $F(v)$ by leveraging $f(u)$ and $F(u)$ from the parent node $u$. Suppose $\delta(u,c)=v$. Intuitively, as the string $\chi_u$ of parent $u$ is a prefix of the string $\chi_v$ of node $v$, it is likely that $F(u)$ and $F(v)$ share some common longest-matching-prefixes in the beginning. It can be proved that when $\chi_v \notin V$, $F(v)$ consists of (1) the tokens from $F(u)$, followed by (2) the longest-matching-prefixes that the procedure from Definition~\ref{def:failure-main} generates for the string $\chi_{f(u)}c$.  Otherwise, when $\chi_v \in V$, it's trivial that $F(v)=[\chi_v]$ based on Definition~\ref{def:failure-main}.  Notice that $f(v)$ and $F(v)$ are computed using similar information for nodes that have strictly smaller depth than $v$.  Breadth-First-Search (BFS) is suitable for the computation.

Algorithm~\ref{alg:init} is the precomputation algorithm. On line~\ref{alg:init-trie}, the algorithm builds a trie for $V$ and keeps track of
$r$ and $r_{\mysi}$. 
These nodes have depth 0 and 
are the starting points for our BFS traversal (line~\ref{alg:init:bfs_queue}).  We assume that initially $f(v)=\nulls$ and $F(v)=\myemptylist$ for every node $v$.  The core part is in lines~\ref{alg:init:core-start}-\ref{alg:init:core-end}, which computes $f(v)$ and $F(v)$ as discussed earlier.

The rest of the algorithm handles technical details.  E.g., if $\mysi$ is the empty string, the nodes $r$ and $r_{\mysi}$ are identical; accordingly, line~\ref{alg:init:bfs_queue} avoids duplicate nodes.  Otherwise, $r_{\mysi}$ is a descendant of $r$, and we need line~\ref{alg:init:skip} to avoid revisiting it in the BFS traversal.

It can be proved that Algorithm~\ref{alg:init} correctly precomputes $f(v), F(v)$ for each trie node $v$.

\begin{algorithm}
\DontPrintSemicolon
\mysetind
\SetKwProg{Fn}{Function}{:}{}
\nonl
\Fn{\textsc{Precompute}{\normalfont (}$V${\normalfont )}}{
  $r,r_{\mysi} \leftarrow \textsc{Buildtrie}(V)$ \label{alg:init-trie}\;
  queue $\leftarrow$ ($r_{\mysi} \neq r$) ? [$r$, $r_{\mysi}$] : [$r$] \label{alg:init:bfs_queue}\;
  \While{\mykw{not} \textsc{Empty}(queue)} {
    $u \leftarrow$ \textsc{Dequeue}(queue) \;
    \For{$c,v$ \mykw{in} \textsc{OutgoingEdges}($u$)} {
        \lIf{$v = r_{\mysi}$}{ \Continue \label{alg:init:skip}}
        \If{$\chi_v \in V$ \label{alg:init:core-start}}{
            $f(v), F(v) \leftarrow r_{\mysi}, [\chi_v]$ \label{alg:init:case1}
        }
        \Else{
            $z, Z \leftarrow$ $f(u)$, $\myemptylist$ \;
            \While{$z \neq \nulls$ \mykw{and} $\delta(z,c) = \nulls$\label{alg:init:while-begin}}{
                $Z \leftarrow $ \textsc{Extend}(Z, $F(z)$) \label{alg:init:z-update}\;
                $z \leftarrow f(z)$ \label{alg:init:while-end}\;
            }
            \If{$z \neq \nulls$}{
                $f(v), F(v) \leftarrow \delta(z,c), F(u)+Z$ \label{alg:init:core-end} \;
            }
        }
        \textsc{Enqueue}(queue, $v$) 
    } 
  }
  \Return $r$ \;
}
\caption{Precomputation}
\label{alg:init}
\end{algorithm}

\subsection{Complexity Analysis}
\label{sec:complexity}

The complexity of tokenization (Algorithm~\ref{alg:inference}) can be proved to be $O(n)$
in a similar way as Aho-Corasick~\cite{aho-corasick-1975-efficient}.
In brief, each step (an iteration of the loop from lines~\ref{alg:inference:step_start}-\ref{alg:inference:goto}) makes zero or more failure transitions followed by exactly one normal (non-failure) transition. In each step, suppose we start at node $u$ with depth $d$.  We never follow more than $d$ failure transitions in that step: each failure transition takes us to a node with a strictly smaller depth.  Any normal transition along trie edges increments the depth $d$ of node $u$ by 1 (line~\ref{alg:inference:goto}). Therefore, the total number of failure transitions is no more than the total number of normal transitions, which is $O(n)$.  Each transition is $O(1)$ plus the work to extend the list of tokens on line~\ref{alg:inference:append}.  As there are at most $n$ resulting tokens in total, the total tokenization time is $O(n)$.

Since at least $n$ operations are required to read the entire input, our $O(n)$ algorithm is asymptotically optimal.  To the best of our knowledge, this is the first time that the optimal complexity for MaxMatch is proved to be strictly $O(n)$, without a vocabulary-specific multiplicative factor.

For precomputation (Algorithm~\ref{alg:init}), the BFS traversal itself is $O(M)$, where $M$ is the sum of the lengths of vocabulary tokens. A similar depth-based analysis (as in the case of the tokenization algorithm) shows that that the total number of times we traverse a failure
link on line~\ref{alg:init:while-end} is $O(M)$. 

The non-trivial parts are the construction of $F(\cdot)$ on lines~\ref{alg:init:z-update} and \ref{alg:init:core-end}. 
The total size of $F(\cdot)$ is $O(Mm)$:
there are $O(M)$ lists, and the size of each list is $O(m)$.  A straightforward implementation needs $O(Mm)$ time and space to construct and store $F(\cdot)$.
This is good enough in practice, as the precomputation is performed offline before any tokenization process.  We plan to discuss optimized implementations in a follow-up publication.

\subsection{Connection with Other Methods / Tasks}
\label{sec:algo-discussions}

\lmmshort{} can be turned into a finite-state transducer (FST) \cite{mohri-1997-finite} by eliminating the failure transitions in Algorithm~\ref{alg:inference}.\footnote{This is analogical to \citet{aho-corasick-1975-efficient} where the Aho-Corasick algorithm can be stated as a deterministic finite-state automaton.} 
An FST extends a finite-state automaton (FSA) with an output tape. 
To turn \lmmshort{} into an FST, for node $u$ and character $c$, we define the state transition function $\delta'(u,c)$ and the output function $\sigma'(u,c)$ as follows:
\begin{itemize}
\item 
$\delta'(u,c)$ precomputes the final state in lines~\ref{alg:inference:failurewhilebegin}-\ref{alg:inference:goto} of Algorithm~\ref{alg:inference}, where it starts from $u$ and follows failure transitions as needed, until it consumes $c$ or meets a null failure link; 
    \item $\sigma'(u,c)$ consists of the failure pops collected along the way.
\end{itemize}
Specially, if the original trie link $\delta(u,c)$ exists, 
according to the above definition, it's obvious that
$\delta'(u,c)=\delta(u,c)$ and $\sigma'(u,c)=[]$. 
Then lines~\ref{alg:inference:failurewhilebegin}-\ref{alg:inference:goto}  in Algorithm~\ref{alg:inference} can be replaced with two statements: $\text{tokens} \leftarrow \textsc{Extend}(\text{tokens}, \sigma'(u, s[i]))$ and $u \leftarrow \delta'(u, s[i])$; the loop (started on line~\ref{alg:inference:step_start}) breaks when $u$ becomes $\nulls$. Hence, \lmmshort{} makes exactly one state transition on each input character.  Obviously, the time complexity is linear, despite more space needed to store precomputed results.

\lmmshort{} extends the Aho-Corasick Algorithm~\citep{aho-corasick-1975-efficient}. It can be applied to more string search or transducer problems. 
Let us name a few here. 
\lmmshort{} can be adapted to solve 
the multi-keyword search problem which Aho-Corasick is designed for. 
It can be also adapted to address other MaxMatch variants, such as Backward MaxMatch~\cite{webster-kit-1992-tokenization}, recognizing unseen characters as single-character tokens~\cite{handbook-nlp}, or combing with transformation rules~\cite{sassano-2014-deterministic}. 
Other potential applications include word segmentation in Asian languages~\cite{sassano-2014-deterministic}, phonological or morphological analysis~\cite{kaplan-1994-phonology,jurafsky-martin-2009-speech}.

\section{Linear-Time End-to-End Tokenization}
\label{sec:e2ewp}

The existing BERT tokenization implementations~\cite{bertoss}
pre-tokenize the input text (splitting it into words by punctuation and whitespace characters) and then call WordPiece tokenization on each resulting word. For example, the text \mytxt{john johanson's} may be split into [\mytxt{john}, \mytxt{johan}, \mytxts{son}, \mytxt{'}, \mytxt{s}].

We propose an end-to-end WordPiece tokenizer that combines pre-tokenization and WordPiece into a single, linear-time pass. It uses the \lmmshort{} trie matching and failure transition loop as much as possible and only checks for punctuation and whitespace characters among the relatively few input characters that are not handled by the loop. It is more efficient as it traverses the input only once, performs fewer punctuation / whitespace checks, and skips the creation of intermediate words.

\paragraph{Precomputation} We use the same process as in Section~\ref{sec:initialization}, with several differences:

After the trie is constructed, we remove all trie links labeled with a punctuation character.\footnote{This may remove links on the path from $r$ to $r_{\mysi}$ when the suffix indicator contains a punctuation; those links were unnecessary: $r_{\mysi}$ is reached only by following failure links.}  Then, for every possible punctuation character $c$, we add a trie data node $v$ with no descendants, and a trie link from the root $r$ to $v$ with label $c$.  If $c$ is part of the vocabulary, we set $\chi_v = c$, otherwise $\chi_v = \myunk$.

The resulting trie matches all punctuation characters, as either themselves or as $\myunk$, depending on the vocabulary.  Punctuation characters are not part of longer tokens, and there is no suffix token for a punctuation character.  This reflects the fact that each punctuation character is a word by itself.

We then run the rest of Algorithm~\ref{alg:init} to compute the failure pops and failure links.

Finally, for punctuation nodes, we set their failure links to a special node $r_p$; their failure pops are not changed. The special node $r_p$ has no parent and no descendants, and $\chi_{r_p}=\varepsilon,f(r_p) = \nulls$.  Node $r_p$ indicates that a punctuation character was matched.

\paragraph{Tokenization} Algorithm~\ref{alg:e2ewp} tokenizes general text into wordpieces.  It starts by appending a whitespace \mywhitespace{} at the end of the input (line~\ref{alg:e2ewp:start}).  In each iteration, it recognizes wordpieces for the current word by employing (almost) the same routine as in single-word tokenization (lines~\ref{alg:e2ewp:routinestart}-\ref{alg:e2ewp:edge_case} in Algorithm~\ref{alg:e2ewp} versus lines~\ref{alg:inference:call}-\ref{alg:inference:edge_case} in Algorithm~\ref{alg:inference}).\footnote{The common routine can be factored out as a function.} 

When returning from \textsc{MatchLoop}(), Algorithm~\ref{alg:e2ewp} must have met a character that cannot be consumed after attempting failure transitions, such as a whitespace, a punctuation, or some unseen character. Lines~\ref{alg:e2ewp:check}-\ref{alg:e2ewp:reset} examine whether the current word can be tokenized (by checking if the current position is at a word boundary and where the node $u$ lands at) and reset the tokens as appropriate (see related discussions in Section~\ref{sec:inference}).

Lines~\ref{alg:e2ewp:check_edge_case}-\ref{alg:e2ewp:edge_case} further handle the corner case that the word happens to be the suffix indicator itself (in the same way as Algorithm~\ref{alg:inference}, see Section~\ref{sec:inference}). 
Note that normally the suffix indicator contains only punctuation characters (e.g., \mytxts{} in BERT); in that case
lines~\ref{alg:e2ewp:check_edge_case}-\ref{alg:e2ewp:edge_case} can be saved, because the suffix indicator itself is not be tokenized as a single word.

The tokens of the current word are then appened to the result (line~\ref{alg:e2ewp:append}).
Finally, the algorithm moves the cursor past the boundary of the current word (lines~\ref{alg:e2ewp:loop3_start}-\ref{alg:e2ewp:loop3_end}) and skips any following whitespaces (lines~\ref{alg:e2ewp:loop4_start}-\ref{alg:e2ewp:loop4_end}) to process the next word.

It can be shown that Algorithm~\ref{alg:e2ewp} is consistent with \citet{bertoss} 
for general text tokenization, and the time complexity is $O(n)$.

\begin{algorithm}
\DontPrintSemicolon
\mysetind
\SetKwProg{Fn}{Function}{:}{}
\nonl
\Fn{\textsc{E2EWordPiece}{\normalfont (\text{text}) }}{

  $\text{result}, \; s, \; i \leftarrow \myemptylist, \; \text{text}\mywhitespace{}, \; 0$\label{alg:e2ewp:start} \;
  \While{${\normalfont i} < |s|$ \label{alg:e2ewp:loop1_start}}{
    $\text{tokens}, u, i \leftarrow  \textsc{MatchLoop}(s, i)$ \label{alg:e2ewp:routinestart} \;
    \If{\mykw{not} $\!\!$ \textsc{\isatwordboundary{}}($s,\! i$) $\!\!$ 
    \mykw{or} $\!u \!\notin \!\{r, \! r_{ \! \mysi}, \! r_p \!\} \!\! $ \label{alg:e2ewp:check}}{
        tokens $\leftarrow [ \myunk ]$ \label{alg:e2ewp:reset}
    }
    \ElseIf{
    $
    u = r_{\mysi}  
      \mykw{ and } |\text{tokens}|=0 $ \label{alg:e2ewp:check_edge_case}}{
        tokens $\leftarrow \textsc{OriginalWordPiece}(\mysi)$ \label{alg:e2ewp:edge_case}
    }
    result $\leftarrow$ \textsc{Extend}(result, tokens) \label{alg:e2ewp:append}\;
    \While{$i\! <\! |s|$ \mykw{and} 
    \mykw{not} \textsc{\isatwordboundary{}}($s,i$)\label{alg:e2ewp:loop3_start}}{
        $i \leftarrow i+1$ \label{alg:e2ewp:loop3_end}
    }
    \While{$i < |s|$ \mykw{and} \textsc{IsSpace}$(s[i])$\label{alg:e2ewp:loop4_start}}{
        $i \leftarrow i+1$\label{alg:e2ewp:loop4_end}
    }
    \label{alg:e2ewp:loop1_end}
  }
  \Return result 
}
\nonl
\Fn{\textsc{\isatwordboundary{}}{\normalfont (}$s$, $i${\normalfont )}}{
    \Return $i\!\geq\!|s|$ \mykw{or} 
    ($i\!>\!0$ \mykw{and} \textsc{IsPunc}$(s[i\!-\!1])$) \mykw{or} \textsc{IsSpace}$(s[i])$ \mykw{or} \textsc{IsPunc}$(s[i])$     \label{alg:e2ewp:check_word_boundary}
}
\caption{End-to-End Tokenization}
\label{alg:e2ewp}
\end{algorithm}

\section{Experiments} %
\label{sec:experiments}

\paragraph{Experimental Setup}

We benchmark our method against two widely-adopted WordPiece tokenization implementations:

\begin{itemize}

\item HuggingFace Tokenizers~\cite{huggingface-whole-package}, from the HuggingFace Transformer library, one of the most popular open-source NLP tools.

\item TensorFlow~Text~\cite{tftext-whole-package}, the official library of text utilities for 
TensorFlow.

\end{itemize}

In both cases, we use pre-tokenization and WordPiece tokenization,  and skip other steps provided by those libraries (text cleanup, normalization, etc) for fair comparison.   Both libraries use the original WordPiece tokenization algorithm~\cite{bertoss}.  They both generate not only the numeric ids of the tokens, but also the token strings and start/end offsets of the input word.  We modify both libraries to generate only the token ids,\footnote{The original TensorFlow~Text first generates the token strings and next looks them up in a dictionary to generate token ids.  For a fair comparison, we adapt it to directly return the token ids, with no intermediate token strings.} for two reasons: (1) most downstream models (e.g., BERT) consume only the token ids, and (2) we want to focus on the core tokenization work, not on, e.g., string copying.

We implement \lmmshort{} and \etoeshort{} and made them return the numeric ids of the tokens, leveraging a double array-based trie 
library~\cite{yata2007237}. 

We compare our algorithms with HuggingFace and TensorFlow~Text on a large corpus (several million words) and found that the tokenization results are identical for both single-word and end-to-end tokenization. In the rest of this section, we focus on the tokenization speed.

All experiments are conducted on a Linux desktop with a six-core Intel Xeon @ 3.60GHz CPU and 64GB memory.  %
We iterate each benchmark (after warming up) until it ran for a long-enough period of time, repeat each experiment 10 times, and report the average results. Our method is implemented and benchmarked in C++; so is TensorFlow~Text. HuggingFace uses (and is benchmarked in) Rust.

We use the WordPiece vocabulary released with the BERT-Base, Multilingual Cased model, a model that supports 104 languages~\cite{bertoss}.

To generate the test data, we sample 1,000 sentences from the multilingual Wikipedia dataset, covering 82 languages including English, Chinese, French, Russian, etc. On average, each word has 4 characters, and each sentence has 82 characters or 17 words.  We found this dataset large enough: a much larger dataset (consisting of hundreds of thousands of sentences) generated similar results. 

We run BERT's 
{\tt BasicTokenizer}~\cite{bertoss}
to clean up and normalize each sentence, including Unicode clean-up and normalization.  Following the 
guidance
for the BERT-Base Multilingual Cased model~\cite{bertoss}, we do not instruct {\tt BasicTokenizer} to do lower casing or accent stripping.  In addition, preprocessing adds spaces around every CJK character, and thus Chinese is effectively character-tokenized.  For simplicity, we keep Chinese in the test set, but keep in mind that each Chinese word is just one Chinese character, and any WordPiece implementation is efficient on such short words.  Using a dataset with long words would emphasize the speed advantage of our algorithm even more than indicated below.

For single-word tokenization, we further used {\tt BasicTokenizer} to pre-tokenize each sentence on punctuation and whitespace characters. This results in 17,223 words, 8,508 of them unique.

\paragraph{Results} Table~\ref{tab:bench_wiki_all} shows the mean and the 95 percentile\footnote{When computing the 95 percentile, the running time on each individual input is approximated by the average running time of all input examples of the same length.\label{foot:pctl}} running time when tokenizing a single word or general text (end-to-end) for each system. 
For single-word tokenization, ours is 3x faster on average; the speedup is greater for long-tail inputs. Regarding general text end-to-end tokenization, ours is 8.2x faster than HuggingFace and 5.1x faster than TensorFlow~Text on average.
Figure~\ref{fig:wiki_by_len} shows how the running time grows with respect to the input length for single-word tokenization. 

\begin{table}[htbp]
\centering
\setlength\tabcolsep{3pt} %
\begin{tabular}{lrrrr}
\toprule
\multirow{2}{*}{System}    & \multicolumn{2}{c}{Single Word} & \multicolumn{2}{c}{End-to-End}\\
\cmidrule(r{0.3pt}){2-3} 
\cmidrule(l){4-5}
                              & mean &  95pctl  
                              & mean &  95pctl  \\
\midrule
HuggingFace     & 274 &   778     & 13,397 &  40,255  \\
TensorFlow~Text & 246 &   622     & 8,247 &   23,507 \\
Ours    &  82 &   139     &  1,629  &   4,400 \\
\bottomrule
\end{tabular}
\caption{The running time of each system in ns.}
\label{tab:bench_wiki_all}
\end{table}

\begin{figure}[htb!]
\begin{center}
\includegraphics[width=0.9\linewidth]{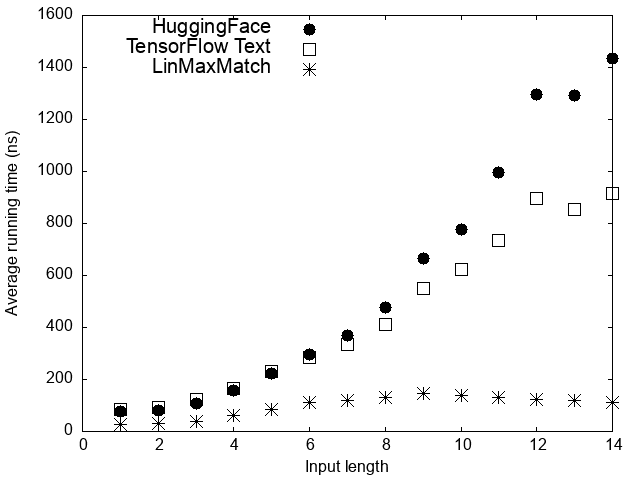}
\caption{Average running time of each system with respect to the input length for single-word tokenization.}
\label{fig:wiki_by_len}
\end{center}
\end{figure}

\section{Conclusion}%
\label{sec:conclusions}

We proposed \lmmshort{} for single-word WordPiece tokenization, which is asymptotically-optimal linear-time with respect to the input length, without a vocabulary-specific multiplicative factor.  We also proposed \etoeshort{} that combines pre-tokenization and WordPiece tokenziation into a single, linear-time pass for even higher efficiency. Experimental results show that our approach is 8.2x faster than HuggingFace and 5.1x faster than TensorFlow~Text on average for general text tokenization. For future work,  we will adapt the proposed methods to more text processing techniques.

\section{Acknowledgements}%
We thank Xiaoxue Zang, Gabriel Schubiner, Max Gubin, Jacob Devlin, Kristina Toutanova, and Ed Chi for discussing and reviewing the work, 
Mike Schuster for clarifying their referred paper, and the anonymous reviewers for their helpful feedback.

\newpage
\bibliographystyle{acl_natbib}
\bibliography{references}

\newpage
\appendix
\section{Mathematical Formulations and Proofs of \lmmshort{}}
\label{sec:appendix-math}
In this section, we present the mathematical formulations of the proposed \lmmshort{} algorithm and prove the correctness.

We introduce more notations here. 
\begin{definition}
\label{def:string_length}
The \emph{length} of string $w$ is $|w|$ (i.e., the number of characters in $w$) if $w$ does not start with $\mysi$; otherwise, its length is $|w|-|{\mysi}|$. 
\end{definition}
For example, the length of \mytxt{abc} is 3, the length of \mytxts{d} is 1 (where \mytxts{} is the suffix indicator), and the length of $\varepsilon$ or $\mysi$ is 0.

\begin{definition}
\label{def:pw}
Given vocabulary $V$, let ${p_w}$ be \emph{the longest non-empty prefix} of $w$ that is in $V$. That is,  
\begin{align*}
p_w \eqdef \arg \max_{w'} \big\{ |w'| \mid & \; w' \text{ is a prefix of } w, \\
                                       & \; w' \in V, w' \notin \{\varepsilon, \mysi{}\} \big\}
\end{align*}
Specially, $p_w \eqdef\varepsilon$ if no such prefixes exist. In addition, if $w$ starts with $\mysi$, the prefix $p_{w}$ should also start with $\mysi$ (unless $p_{w}$ is empty).\footnote{For example, suppose that the suffix indicator is \mytxts{}, and \mytxt{\#} (a single character) is in $V$ but \mytxts{a} is not in $V$. Then by definition $p_{\mytxts{a}}$ is not $\mytxt{\#}$ (the character); it is $\varepsilon$ instead.} When $w=\mysi$, $p_{\mysi}\eqdef\varepsilon$ for clarity. 
\end{definition}

\begin{definition}
\label{def:qw}

Let $q_w$ be \emph{the suffix} of $w$ after replacing the prefix $p_w$ with $\mysi$. That is, if $w=p_w w''$, $q_w \eqdef \mysi w''$. 
\end{definition}

For example, if $V=\{\mytxt{a}, \mytxt{ab}, \mytxts{c}\}$, let the suffix indicator $\mysi$ be \mytxts{}, then 
$p_{\mytxt{abcd}}=\mytxt{ab}$, $q_{\mytxt{abcd}}=\mytxts{cd}$,
$p_{\mytxts{cd}}=\mytxts{c}$, and $q_{\mytxts{cd}}=\mytxts{d}$.
We see that if $w \in V$, $p_w=w$ and $q_w=\mysi$. 

\begin{lemma} \label{thm:pwc}
For an nonempty string $wc$, where $c$ is the last character and $w$ is the prefix ($w$ could be $\varepsilon$ or $\mysi$), if $wc \notin V$, we have $p_{wc}=p_w$ and $q_{wc}=q_wc$.
\end{lemma}
\begin{proof}
First, we prove that $p_{wc}$ does not include the last character $c$ by contradiction. Let's suppose that $p_{wc}$ includes the last character $c$. Then $p_{wc}=wc$ (since $p_{wc}$ is a prefix of $wc$). Because $p_{wc}\in V$ (Definition~\ref{def:pw}), $wc \in V$, which contradicts that $wc \notin V$. 

Now, because $p_{wc}$ does not include the last character $c$, it is obvious that $p_{wc}=p_w$.

Next, let $w=p_w w''$, then $q_w=\mysi{} w''$ (Definition~\ref{def:qw}). Since $p_{wc}=p_w$, we have $wc=p_w w''c=p_{wc}w''c$. Therefore, $q_w=\mysi{} w''c=q_w c$. 
\end{proof}

Let $\gamma_{w}$ denote the trie node that represents the string $w$ (so $\chi_{\gamma_w} = w$), or $\nulls$ if no such nodes exist. When $\gamma_w \neq \nulls$, we say the string $w$ is \emph{on the trie}. For the example in Figure~\ref{fig:example}, $\gamma_{\mytxt{abcd}}$ is the node 6 while $\gamma_{\mytxt{abcdz}}=\nulls$.

Table~\ref{tab:notation2} summarizes the additional notations.

\begin{table}[htb!]
    \centering
\newcommand\myoptional[1]{#1}
\renewcommand\myoptional[1]{}   %
\setlength\tabcolsep{2pt} %
\begin{tabular}{ll}
\textbf{Symbol} \myoptional{& \textbf{Section}} & \textbf{Meaning} \\
\hline
    $p_w$ \myoptional{& \ref{sec:problem}} & The longest prefix of $w$ being in $V$ \\
    $q_w$ \myoptional{& \ref{sec:problem}} & The suffix of $w$ after removing prefix \\
    \myoptional{&} & $p_w$, plus a preceding $\mysi$ \\
    $M(w)$ \myoptional{& \ref{sec:problem}} & MaxMatch result for $w$ given $V$ \\
    $\gamma_w$ \myoptional{& \ref{sec:concepts}} & The node that represents string $w$ \\ 
    $g(w)$ \myoptional{& \ref{sec:concepts}} & MinPop Matching $w$ onto some node \\ 
    $G(w)$ \myoptional{& \ref{sec:concepts}} & Tokens popped when computing $g(w)$ \\ 
    $h(u,c)$ \myoptional{& \ref{sec:inference}} & $g(\chi_{u}c)$ (or $\nulls$ if $u=\nulls)$ \\ 
    $H(u,c)$ \myoptional{& \ref{sec:inference}} & $G(\chi_{u}c)$ (or $\myemptylist$ if $u=\nulls$)\\ 
\end{tabular}
\caption{Additional Notations (continued from Table~\ref{tab:notation})}
\label{tab:notation2}
\end{table}

\subsection{MaxMatch in WordPiece}

MaxMatch in WordPiece tokenization~\cite{bertoss}
can be formalized as follows:\footnote{Excluding the corner case where $w=\mysi$; see discussions.}

\begin{definition}{\textbf{MaxMatch}}
\label{def:maxmatch}

Given vocabulary $V$, for string $w$, MaxMatch $M(w)$ is recursively defined as:
\begin{equation} \label{equ:mm}
M(w) \! \eqdef \!
\begin{cases}
\myemptylist & \text{if } w = \varepsilon \text{ or } \mysi, \\
[\myunk] & \text{elif } p_w = \varepsilon, \\
[\myunk] & \text{elif } M(q_w) = [\myunk], \\
[p_w]\!+\!M(q_w) \!\!\! & \text{otherwise}. 
\end{cases}
\end{equation}
\end{definition}

Note that if the input is exactly the suffix indicator $\mysi$ itself, by Definition~\ref{def:maxmatch}, $M(\mysi)\eqdef\myemptylist$, which may be different from the original MaxMatch algorithm~\cite{bertoss} (see Sec.~\ref{sec:inference}).  Throughout this section, we focus on Definition~\ref{def:maxmatch}, but be aware that if the original input is exactly the suffix indicator, we resort to the original MaxMatch algorithm.

\subsection{MinPop Matching}

We introduce a few concepts and discuss their properties and relationships, as shown in Figure~\ref{fig:concepts}, which eventually lead to the mathematical formulation of the algorithm and the proofs.

\begin{figure}[htb!]
\begin{center}
\includegraphics[width=0.9\linewidth]{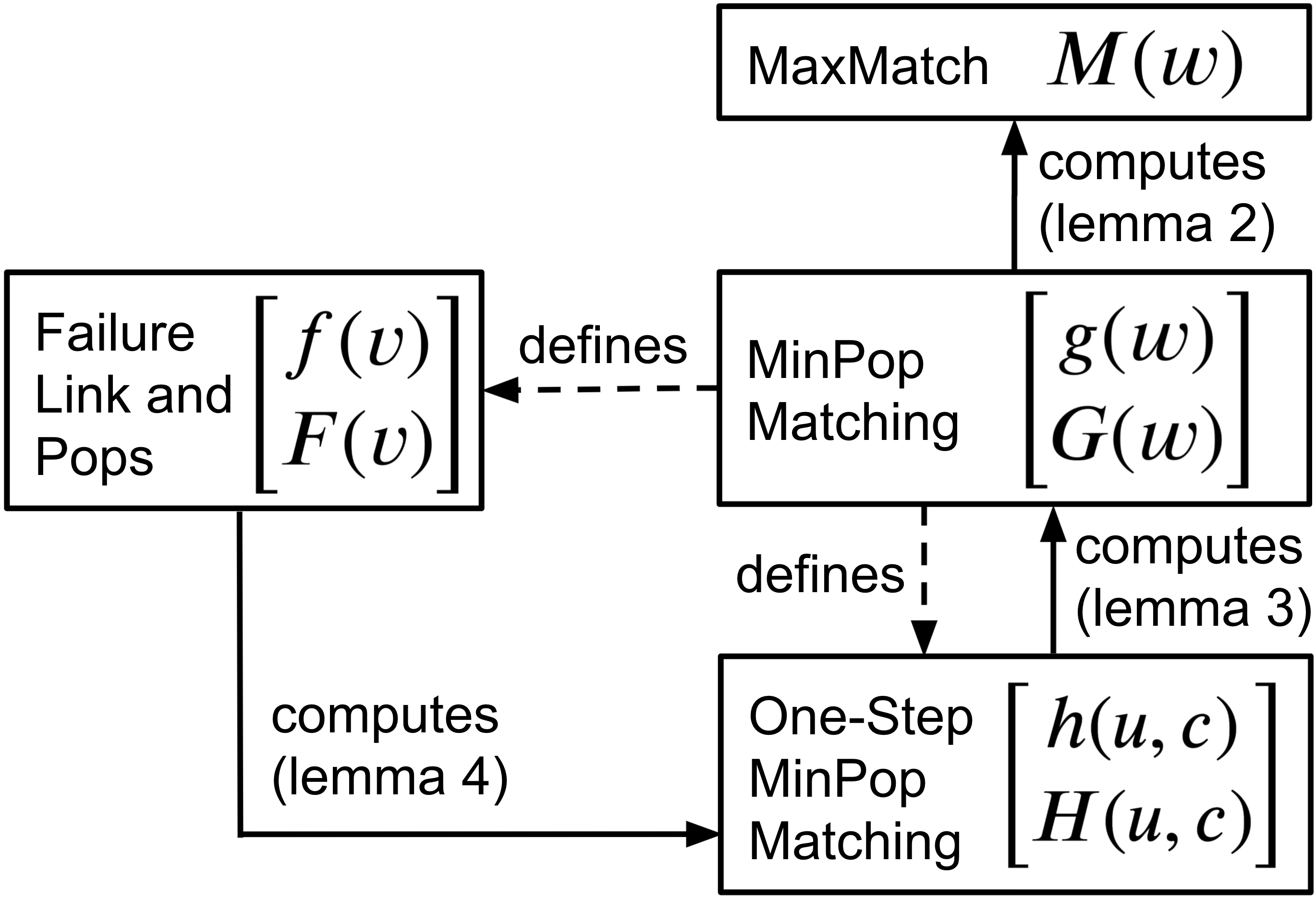}
\caption{Definitions and the relationships.}
\label{fig:concepts}
\end{center}
\end{figure}

The first concept is \textbf{MinPop Matching}, which means "\textbf{minimally popping} longest-matching prefixes off the beginning of a string until \textbf{matching} a trie node". 
The formal definition is as follows:

\begin{definition}{\textbf{MinPop Matching}}
\label{def:grounding}

For a string $w$, define:

\begin{itemize}
    \item 
    $g(w)$: returns a node that represents $w$ if possible, or a node pointing to the suffix of $w$ after popping the least number of consecutive prefixes following the left-to-right longest-match-first process if possible, otherwise $\nulls$.
\item
$G(w)$: returns the list of consecutive longest-matching prefix tokens that are popped when computing $g(w)$.
\end{itemize}

\begin{equation} \label{equ:ground}
    \begin{bmatrix}
        g(w) \\ G(w)
    \end{bmatrix}
    \!\eqdef \!
    \begin{cases}
        \!\begin{bmatrix}
            \gamma_w \\ [\;]
        \end{bmatrix} 
        & \text{if } \gamma_w \neq \nulls, \\[4mm]
        \!\begin{bmatrix}
            \nulls \\ [\;]
        \end{bmatrix} 
        &\text{elif } p_w = \varepsilon, \\[4mm]
        \!\begin{bmatrix}
            g(q_w) \\
            [p_w] \!+\!G(q_w)
        \end{bmatrix} 
        & \text{otherwise.} \\
    \end{cases}
\end{equation}
\end{definition}

\begin{example} \label{exm:ground}
Table ~\ref{tab:grounding} shows $g(w)$ and $G(w)$ of example strings using the vocabulary in Figure~\ref{fig:example}. 
\end{example}

\begin{table}[htb]
    \centering
\setlength\tabcolsep{3pt} %
\begin{tabular}{c|cccccc}
    $\pmb{w}$  & \mytxt{abcd} &  \mytxts{bcd} & \mytxts{cdz} & \mytxts{bcdz} & \mytxt{z}\\
    \hline
    $\pmb{g(w)}$ & 6 & 10 & 13 & 13 & $\nulls$\\
    $\pmb{G(w)}$  &  $\myemptylist$ &  [\mytxts{b}] & [\mytxts{c}] & [\mytxts{b}, \mytxts{c}] & []\\
\end{tabular}
\caption{Examples of $g(w)$ and $G(w)$ for Figure~\ref{fig:example}}
\label{tab:grounding}
\end{table}

Note that if $w$ is on the trie, no popping is needed when computing $g(w)$ and $G(w)$. See Example~\ref{exm:ground}.

MinPop Matching provides an alternative way to compute MaxMatch as shown in Lemma~\ref{thm:ground}.
\begin{lemma} \label{thm:ground}
For ease of presentation, we augment the trie by adding two nodes representing $\mywhitespace{}$ and $\mysi\mywhitespace{}$, respectively, where \mywhitespace{} is the whitespace character that is not in the alphabet of the vocabulary. Note that although $\mywhitespace{}$ and $\mysi\mywhitespace{}$ are on the trie, the two strings are not added to the vocabulary. Figure~\ref{fig:trie_augmented} shows the augmented trie built from the example vocabulary in Figure~\ref{fig:example}. Then MaxMatch {\normalfont $M(w)$} can be equivalently computed as: \begin{equation}
\label{equ:lemma1}
M(w) = 
\begin{cases}
[\myunk] & \text{if } g(w\mywhitespace{}) = \nulls, \\
G(w\mywhitespace{}) & \text{otherwise}.
\end{cases}
\end{equation}
\end{lemma}

\begin{proof}
If $w$ is either $\varepsilon$ or $\mysi$, it's straightforward that $g(w\mywhitespace{})$ is $\gamma_{\mywhitespace{}}$ or $\gamma_{\mysi\mywhitespace}$, which is not $\nulls$ on the augmented trie, and $G(w\mywhitespace{})=[]=M(w)$.

Let $w \notin \{ \varepsilon, \mysi \}$. Since $\mywhitespace$ is not in the vocabulary alphabet, $w\mywhitespace{}$ is not on the trie (i.e., $\gamma_{w\mywhitespace{}}=\nulls$).

If $w$ can be successfully tokenized, according to Equation~\ref{equ:ground}, it will keep popping the longest-matching prefixes until the remaining suffix becomes $\mysi\mywhitespace{}$, which is on the augmented trie. Hence, $g(w\mywhitespace{})$ becomes $\gamma_{\sharp\mywhitespace{}}$ ($\neq \nulls$), and $G(w\mywhitespace{})$ equals to $M(w)$. 

Otherwise, by Equation~\ref{equ:ground}, at some point $p_w$ will be $\varepsilon$; thus, $g(w\mywhitespace{})$ will eventually be $\nulls$. Equation~\ref{equ:lemma1} returns $[\mytxt{<unk>}]$, which equals to $M(w)$.
\end{proof}

\begin{figure*}[htb!]
\begin{center}
\begin{subfigure}[b]{0.48\textwidth}
\setlength\tabcolsep{6pt} %
\begin{tabular}{|c|}
\hline
$\pmb{V}$: $\left\{ \mytxt{a}, \, \mytxt{abcdx},  \, \mytxts{b},  \, \mytxts{c}, \, \mytxts{cdy}, \,  \mytxts{dz} \right\}$ \\
\hline
\end{tabular}
\includegraphics[width=1.0\linewidth]{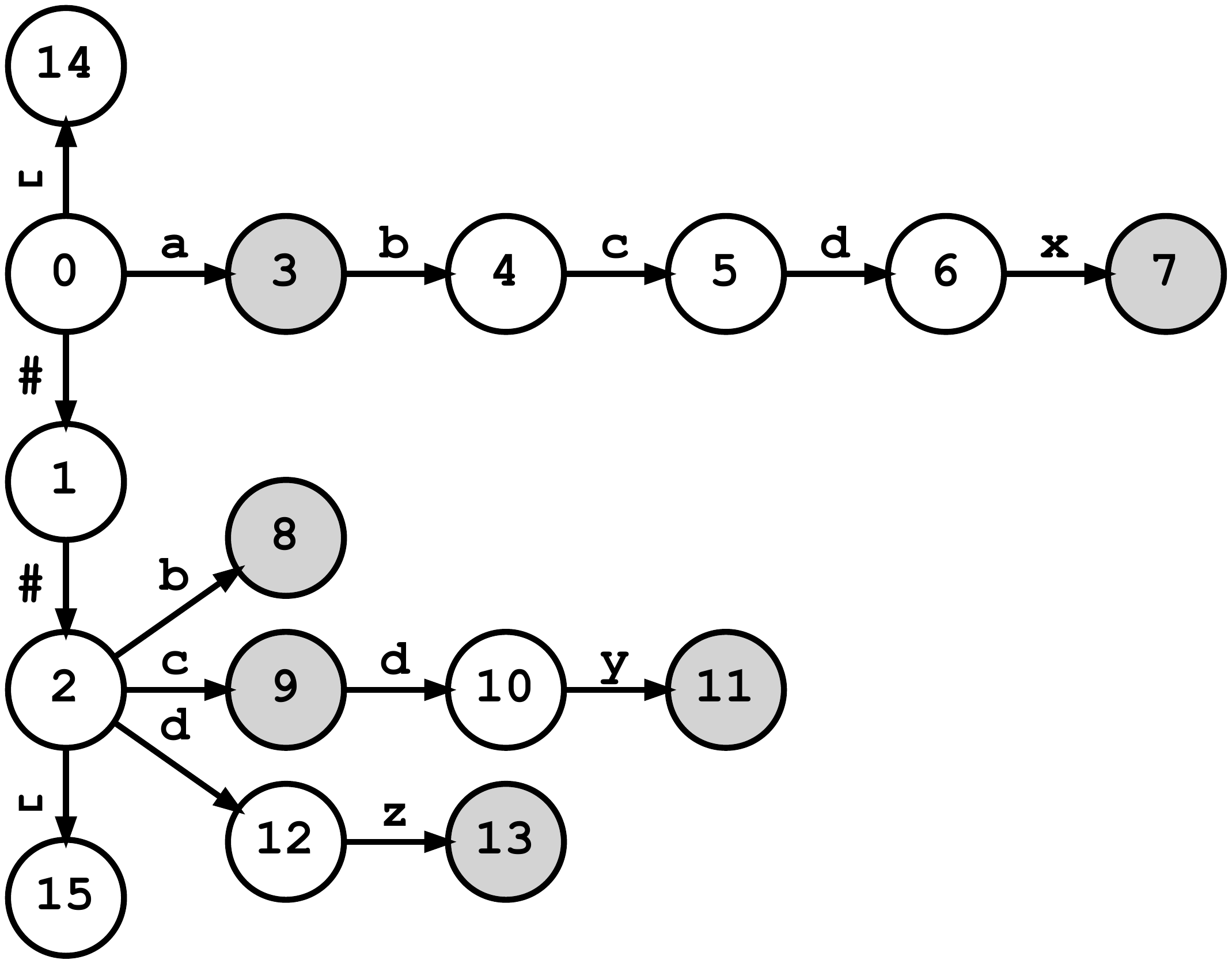}\llap{\includegraphics[width=0.4\linewidth]{legend_trie}}
\caption{The vocabulary and the augmented trie.}
\label{fig:graph3_augmented}
\end{subfigure}
\quad
\begin{subfigure}[b]{0.48\textwidth}
\setlength\tabcolsep{1pt} %
\begin{tabular}{c|ccccccccccc}
    $\pmb{v}$&  \textbf{0} & \textbf{1} & \multicolumn{2}{c}{\textbf{2}} & \multicolumn{3}{c}{\textbf{14}} & \multicolumn{2}{c}{\textbf{15}} \\
    \hline
    $\pmb{F(v)}$& $\myemptylist$ & $\myemptylist$ & \multicolumn{2}{c}{$\myemptylist$} & \multicolumn{3}{c}{$\myemptylist$} & \multicolumn{2}{c}{$\myemptylist$} \\
    $\pmb{f(v)}$& $\nulls$ & $\nulls$ & \multicolumn{2}{c}{$\nulls$} & \multicolumn{3}{c}{$\nulls$} & \multicolumn{2}{c}{$\nulls$} \\
    \hline
    \hline
    $\pmb{v}$&  \textbf{3} & \textbf{4} & \multicolumn{3}{c}{\textbf{5}} & \multicolumn{3}{c}{\textbf{6}}  & \multicolumn{3}{c}{\textbf{7}} \\
    \hline
    $\pmb{F(v)}$& [\mytxt{a}] & [\mytxt{a}] & \multicolumn{3}{c}{[\mytxt{a}, \mytxts{b}]} & \multicolumn{3}{c}{[\mytxt{a}, \mytxts{b}]} & \multicolumn{3}{c}{[\mytxt{abcdx}]} \\
    $\pmb{f(v)}$& 2 & 8 & \multicolumn{3}{c}{9} & \multicolumn{3}{c}{10} & \multicolumn{3}{c}{2} \\
    \hline
    \hline
    $\pmb{v}$ & \textbf{8} & \textbf{9} &  \multicolumn{2}{c}{\textbf{10}} & \multicolumn{3}{c}{\textbf{11}} & \multicolumn{2}{c}{\textbf{12}} & \multicolumn{2}{c}{\textbf{13}}  \\
    \hline
    $\pmb{F(v)}$ & [\mytxts{b}] & [\mytxts{c}] & \multicolumn{2}{c}{[\mytxts{c}]} & \multicolumn{3}{c}{[\mytxts{cdy}]} & \multicolumn{2}{c}{$\myemptylist$} & \multicolumn{2}{c}{[\mytxts{dz}]}  \\
    $\pmb{f(v)}$ & 2 & 2 & \multicolumn{2}{c}{12} & \multicolumn{3}{c}{2} & \multicolumn{2}{c}{$\nulls$} & \multicolumn{2}{c}{2}\\
\end{tabular}
\caption{Complete table of $f(v)$ and $F(v)$. }
\label{fig:F_augmented}
\end{subfigure}
\end{center}
\caption{Augmented trie of the same example vocabulary and the table of failure links and failure pops. Compared to Figure~\ref{fig:example}, nodes 14 and 15 are added representing $\mytxt{\mywhitespace{}}$ and $\mytxts{\mywhitespace{}}$. By Definition~\ref{def:failure-main} we see that for the added nodes (14 and 15), the failure links are $\nulls$ and failure pops are $\myemptylist$. The remaining entries of $F(\cdot)$, $f(\cdot)$ remain the same. }
\label{fig:trie_augmented}
\end{figure*}

\begin{example}
In Figure~\ref{fig:trie_augmented},
$M(\mytxt{abcdx})=G(\mytxt{abcdx}\mywhitespace{})=[\mytxt{abcdx}]$ since $g(\mytxt{abcdx}\mywhitespace{})$ is node 15 ($\neq \nulls$). 
$M(\mytxt{z})=[\myunk]$ since $g(\mytxt{z}\mywhitespace{})=\nulls$.
\end{example}

\subsection{One-Step MinPop Matching}\label{sec:one-step-minpop-matching}

Given that MaxMatch $M(w)$ can by computed via MinPop Matching (Lemma~\ref{thm:ground}), we now discuss how to efficiently compute $g(w)$ and $G(w)$ via the concept of \textit{One-Step MinPop Matching}.

\begin{definition}{\textbf{One-Step MinPop Matching}}
\label{def:one-step-grounding}

$h(u,c)$ and $H(u,c)$ capture this process: from node $u$, \textbf{match one} character $c$ by \textbf{minimally popping} longest-matching prefixes. Mathematically:
\begin{equation}
\label{equ:ground_uc}
    \begin{bmatrix}
        h(u,c) \\ H(u,c)
    \end{bmatrix}
    \!\eqdef\!
    \begin{cases}
    \begin{bmatrix}
        \nulls \\ \myemptylist
    \end{bmatrix}
        & \text{if } u = \nulls, \\[4mm]
    \begin{bmatrix}
        g(\chi_{u}c) \\ G(\chi_{u}c)
    \end{bmatrix}
        & \text{otherwise}
    \end{cases}
\end{equation}
\end{definition}

\begin{example}
Table~\ref{tab:one-step-grounding} shows some example values of $h(u,c)$ and $H(u,c)$ for Figure~\ref{fig:trie_augmented}. 
\end{example}

\begin{table}[htb]
    \centering
\newcommand\myoptional[1]{#1}
\renewcommand\myoptional[1]{}   %
\setlength\tabcolsep{1pt} %
\begin{tabular}{c|cccccc}
    $\pmb{u}$&  8 & 9 & 10 & 13 & 6 & 0\\
    $\pmb{c}$& \mytxt{c} & \mytxt{d} & \mytxt{z} & $\mywhitespace{}$ & \mytxt{z} & \mytxt{z} \\
    \hline
    $\pmb{h(u,c)}$&  9 & 10 & 13 & 14 & 13 & $\nulls$\\
    $\pmb{H(u,c)}$&  [\mytxts{b}] & $\myemptylist$ & [\mytxts{c}] & [\mytxts{dz}] & [\mytxt{a}, \mytxts{b}, \mytxts{c}] &  []
\end{tabular}
\caption{Examples of $h(u,c)$ and $H(u,c)$ for Figure~\ref{fig:trie_augmented}}
\label{tab:one-step-grounding}
\end{table}

Lemma~\ref{thm:induction} shows how to compute $g(w)$ and $G(w)$ efficiently using $h(u,c)$ and $H(u,c)$.
\begin{lemma}
\label{thm:induction}
MinPop Matching $g(\cdot), G(\cdot)$ can be computed recursively as follows: 

If the string is either $\varepsilon$ or $\mysi$, $g(\varepsilon)=r$, $G(\varepsilon)=\myemptylist$; $g(\mysi)=r_{\mysi}$, $G(\mysi)=\myemptylist$  (Definition~\ref{def:grounding}). 

Otherwise, the string contains at least one character. Let's denote the string as $wc$, where $w$ is its prefix and $c$ is the last character. ($w$ could be $\varepsilon$ or $\mysi$). Let $u = g(w)$, we have:

\begin{equation}
\label{equ:ground_induction}
    \begin{bmatrix}
        g(wc) \\ G(wc)
    \end{bmatrix}
    \!=\!
    \begin{bmatrix}
        h(u,c) \\ G(w) + H(u,c)
    \end{bmatrix} \\
\end{equation}
\end{lemma}
\begin{proof}
We prove by induction on the length of the prefix string $w$. Note that the length of a string does not count the leading suffix indicator (Definition~\ref{def:string_length}). 

The basis is when the length of $w$ is 0, i.e., $w$ is either $\varepsilon$ or $\mysi$. It's trivial to verify that Equation~\ref{equ:ground_induction} holds for the basis case.

For the inductive steps, let the length of $w$ be $k$ ($\geq 1$). Assume that Equation~\ref{equ:ground_induction} holds for any string $w'$ and character $c$ where the length of $w'$ is smaller than $k$. There are three cases to discuss.

 \textbf{Case 1}. $\gamma_{w}\neq \nulls$. In this case, $u=g(w)=\gamma_{w}\neq \nulls$, and $\chi_u=w$, $G(w)=\myemptylist$. By Definition~\ref{def:one-step-grounding}, 
\[
\mybm{g(wc)}{G(wc)}\!\!=\!\!\mybm{g(\chi_uc)}{G(\chi_uc)}\!\!=\!\! \mybm{h(u,c)}{H(u,c)}\!\! =\!\!\mybm{h(u,c)}{G(w)\!+\!H(u,c)}\!\!.
\]

In the remaining two cases, $\gamma_{w}=\nulls$, hence $\gamma_{wc}=\nulls$, %
which means $wc \notin V$. Hence, $p_{wc}=p_{w}$ and $q_{wc}=q_wc$ (Lemma~\ref{thm:pwc}). When computing $g(wc)$ and $G(wc)$, since $\gamma_{wc}=\nulls$, by Equation~\ref{equ:ground}, there are two remaining cases:

 \textbf{Case 2}. $\gamma_{w}=\nulls$ and $p_{wc}=\varepsilon$. We have $p_{w}=p_{wc}=\varepsilon$, so by Equation~\ref{equ:ground} $g(w)=g(wc)=\nulls$ and $G(w)=G(wc)=\myemptylist$. Since $u=g(w)=\nulls$, by Equation~\ref{equ:ground_uc} $h(u,c)=\nulls$ and $H(u,c)=\myemptylist$. Hence,
\[
\mybm{g(wc)}{G(wc)}=\mybm{\nulls}{\myemptylist}=\mybm{h(u,c)}{G(w)+H(u,c)}
\]
    
 \textbf{Case 3}. $\gamma_{w}=\nulls$ and $p_{wc}\neq \varepsilon$. Since $p_{wc}=p_w$, we have $q_{wc}=q_wc$, and $g(q_w)=g(w)=u$. Since $q_{w}$ is a shorter string whose length is smaller than $k$, by the induction assumption, we have 
\begin{align*}
\mybm{g(q_{wc})}{G(q_{wc})}&=\mybm{g(q_wc)}{G(q_wc)}=\mybm{h(u,c)}{G(q_w)+H(u,c)}\hspace{-1.5em} &
\intertext{Hence,}
\mybm{g(wc)}{G(wc)} &= \mybm{g(q_{wc})}{[p_{wc}]+G(q_{wc})} & \text{(Eq.~\ref{equ:ground})}\\
                    &= \mybm{h(u,c)}{[p_w] + G(q_w) + H(u,c)} \\
                    &= \mybm{h(u,c)}{G(w) + H(u,c)} & \text{(Eq.~\ref{equ:ground})} 
\end{align*}

Therefore, Equation~\ref{equ:ground_induction} is proved.
\end{proof}

\begin{example}
Take Figure~\ref{fig:trie_augmented} as an example, let $w=\mytxts{bcd}$ and $c=\mytxt{z}$. We know $u=g(w)$ is node 10 and $G(w)=[\mytxts{b}]$ (Table~\ref{tab:grounding}). Given $u=10$ and $c=\mytxt{z}$, we also know that $h(u,c)=h(10,\mytxt{z})=13$ and $H(u,c)=H(10, \mytxt{z})=[\mytxts{c}]$ (Table~\ref{tab:one-step-grounding}). For $wc=\mytxts{bcdz}$, we can see that $g(wc)=h(u,c)=13$, and $G(wc)=G(w)+H(u,c)=[\mytxts{b}]+[\mytxts{c}] = [\mytxts{b}, \mytxts{c}]$.
\end{example}

If we precompute and store $h(u,c), H(u,c)$ for every pair of node $u$ and character $c$, then for an arbitrary string $w$, we can efficiently compute $g(w\mywhitespace), G(w\mywhitespace)$ (Lemma~\ref{thm:induction}) and MaxMatch $M(w)$ (Lemma~\ref{thm:ground}). This results in an algorithm that can be formualized as a finite-state transducer (FST) (more discussions in Section~\ref{sec:appendix_fst}). However, it needs more space to store the $h(u,c)$ and $H(u,c)$ tables. For example, the size of the $h(u,c)$ table is $O(|T|\cdot|\Sigma|)$, where $|T|$ is the size of the trie and $|\Sigma|$ is the size of the alphabet. 

In the following sections, we show that, while maintaining the overall linear time complexity (Section~\ref{sec:complexity}), failure links $f(v)$ and failure pops $F(v)$ can be used to efficiently compute $h(u,c)$ and $H(u,c)$, but with much less space. For example, the size of $f(v)$ table is $O(|T|)$, which is much less than the $O(|T|\cdot|\Sigma|)$ space needed for the $h(u,c)$ table. This eventually results in Algorithm~\ref{alg:inference}, which is a more practical approach.

\subsection{Failure links and Failure Pops}

The mathematical definition of $f(v)$ and $F(v)$ is:

\begin{definition}{\textbf{Failure links and pops}} (continued from Definition~\ref{def:failure-main}).
\label{def:failure}
Mathematically, let $w=\chi_v$, $f(v)$ and $F(v)$ are defined as follows:
\begin{equation}
\label{equ:failure}
    \begin{bmatrix}
        f(v) \\ F(v)
    \end{bmatrix}
    \!\eqdef \!
    \begin{cases}
        \!\begin{bmatrix}
            \nulls \\ [\;]
        \end{bmatrix} 
        & \text{if } p_{w} = \varepsilon, \\[4mm]
        \!\begin{bmatrix}
            g\big(q_{w}\big) \\
            \big[p_{w}\big] \!+\!G\big(q_{w}\big)
        \end{bmatrix} \!\!\!\!\!
        & \text{otherwise.} \\
    \end{cases}
\end{equation}
\end{definition}

Lemma~\ref{thm:ground_uc_induction} shows how to compute $h(u,c)$ and $H(u,c)$ recursively based on $f(\cdot)$ and $F(\cdot)$. 

\begin{lemma} 
\label{thm:ground_uc_induction}
One-Step MinPop Matching $h(u,c)$ and $H(u,c)$ can be computed recursively as follows:
\begin{align}
\label{equ:ground_uc_induction}
    \begin{bmatrix}
        h(u,c) \\ H(u,c)
    \end{bmatrix}
    \!&=\!
    \begin{cases}
        \!\begin{bmatrix}
            \nulls \\ \myemptylist
        \end{bmatrix} 
        & {\normalfont \text{if }} u = \nulls, \\[4mm]
        \!\begin{bmatrix}
            \delta(u,c) \\ \myemptylist
        \end{bmatrix} 
        &{\normalfont \text{elif }} \delta(u,c) \neq \nulls, \\[4mm]
        \!\begin{bmatrix}
            h\big(f(u),c\big)  \\
            F(u) \!+\!H\big(f(u),c\big)
        \end{bmatrix} \!\!\!\!\!\!
        & {\normalfont \text{otherwise.}} \\
    \end{cases}
\end{align}
\end{lemma}

\begin{proof}
The first two rows of Equation~\ref{equ:ground_uc_induction} hold obviously. Now we prove the third row, where $u\neq \nulls$ and $\delta(u,c)=\nulls$. Let $w=\chi_u$. Since $\delta(u,c)=\nulls$, $\gamma_{wc}=\nulls$, or $wc \notin V$. Hence, $p_{wc}=p_{w}$ and $q_{wc}=q_{w}c$ (Lemma~\ref{thm:pwc}). There are two cases to discuss.

\textbf{Case 1}. If $p_{w}=p_{wc}=\varepsilon$, we have $f(u)=\nulls$ and $F(u)=\myemptylist$ (Equation~\ref{equ:failure}). Hence $h(f(u), c)=\nulls$ and $H(f(u), c)=\myemptylist$ (Equation~\ref{equ:ground_uc}).
On the other hand, since $\gamma_{wc}=\nulls$ and $p_{wc}=\varepsilon$,  by Equation~\ref{equ:ground} $g(wc)=\nulls$ and $G(wc)=\myemptylist$. So we have
\[
\mybm{h(u,c)}{H(u,c)}=\mybm{\nulls}{\myemptylist}=\mybm{h(f(u),c)}{F(u)+H(f(u),c)} \\
\]

\textbf{Case 2}. Otherwise, $p_{w}=p_{wc}\neq\varepsilon$, we have
\begin{align*}
\mybm{h(u,c)}{H(u,c)}\!&=\!\mybm{g(wc)}{G(wc)}          & \text{(Eq.~\ref{equ:ground_uc})} \\
          &=\!\mybm{g(q_{wc})}{[p_w] + G(q_{wc})}      & \text{(Eq.~\ref{equ:ground})} \\
          &=\!\mybm{g(q_wc)}{[p_w] + G(q_wc)}        & \text{($q_{wc}\!=\!q_wc$)} \\
          &=\!\mybm{h(g(q_w),c)}{[p_w] + G(q_w) + H(g(q_w),c)}\hspace{-2.5em}   & \text{(Eq.~\ref{equ:ground_induction})} \\
          &=\!\mybm{h(f(u),c)}{F(u) + H(f(u),c)}       & \text{(Eq.~\ref{equ:failure})} 
\end{align*}

Therefore, Equation~\ref{equ:ground_uc_induction} is proved.
\end{proof}

\begin{example}
In Figure~\ref{fig:trie_augmented}, let $u$ be node 6 and $c=\mytxt{z}$, 
\begin{align*}
h(u,c)&=h(f(u),c)=h(10, \mytxt{z})=13 \\ 
H(u,c)&=F(u)+H(f(u),c)=F(6)+H(10,\mytxt{z})\\
      &=[\mytxt{a}, \mytxts{b}] + [\mytxts{c}] = [\mytxt{a}, \mytxts{b}, \mytxts{c}]
\end{align*}
\end{example}

\subsection{Tokenization and its Correctness}
\label{sec:tokenization_correctness}
In this section, we show that Algorithm~\ref{alg:inference} correctly computes MaxMatch $M(w)$ (Definition~\ref{def:maxmatch}) following Lemma~\ref{thm:ground}-\ref{thm:ground_uc_induction}.\footnote{Assume that the original input string $w$ is not the suffix indicator $\mysi$ itself. See Section~\ref{sec:inference}.}

Given string $s$, if call \textsc{MatchLoop}($s$, 0) (Algorithm~\ref{alg:inference}), lines~\ref{alg:inference:failurewhilebegin}-\ref{alg:inference:goto} compute $h(u, s[i])$ and $H(u, s[i])$ based on Lemma~\ref{thm:ground_uc_induction}, while lines~\ref{alg:inference:step_start}-\ref{alg:inference:iplus1} 
compute $g(s)$ and $G(s)$ incrementally based on Lemma~\ref{thm:induction}.
In particular, if $g(s) \neq \nulls$, the resulted 
\emph{tokens} and $u$ are $G(s)$ and $g(s)$, respectively.

We now prove the correctness of Algorithm~\ref{alg:inference} based on Lemma~\ref{thm:ground}. There are two cases to discuss. 

If the input $w$ can be tokenized, according to Lemma~\ref{thm:ground}, 
when running \textsc{MatchLoop}($w\mywhitespace$, 0) on the augmented trie, it will return with $u=g(w\mywhitespace)\in \{\gamma_{\mywhitespace}, \gamma_{\mysi\mywhitespace}\} (\neq \nulls)$ and $\emph{tokens}=G(w\mywhitespace)=M(w)$. 
Analogically, if running \textsc{MatchLoop}($w\mywhitespace$, 0) on the original trie, it would follow the same behavior as on the augmented trie until $i=|w|$ and $u \in \{r, r_{\mysi}\}$. Then the function breaks and returns (since $\delta(u, \mywhitespace)=\nulls$ and $f(u)=\nulls$)). The collected \emph{tokens} are the same as $G(w\mywhitespace)$ on the augmented trie, which is equal to MaxMatch $M(w)$. This is returned as the final output in Algorithm~\ref{alg:inference} (line~\ref{alg:inference:return}).

Otherwise, $g(w\mywhitespace{}) = \nulls$ on the augmented trie (Lemma~\ref{thm:ground}). If running on the augmented trie, \textsc{MatchLoop}($w\mywhitespace, 0$) will break at line~\ref{alg:inference:matchloopbreak} when $f(u) = \nulls$, and in the outputs $i<|w|$ or $u \notin \{r, r_{\mysi}\}$. Now, when running \textsc{MatchLoop}($w\mywhitespace, 0$) on the original trie, it would follow the same behavior and return with the same outputs. Therefore, Algorithm~\ref{alg:inference} returns $[\myunk{}]$ as expected (line~\ref{alg:inference:reset}).

\subsection{Precomputation and its Correctness}

Algorithm~\ref{alg:init} precomputes failure links $f(\cdot)$ and failure pops $F(\cdot)$ based on the following lemma.

\begin{lemma}
\label{thm:induction-f}
The following process correctly computes $f(v),F(v)$ for any trie node $v$.
If $v\in \{r, r_{\mysi} \}$, $f(v)=\nulls$, $F(v)=\myemptylist$ ( Definition~\ref{def:failure}).

Otherwise, let $u$ be the parent of $v$ and $c$ be the label from $u$ to $v$ (i.e., $\delta(u,c)=v$), we have:

\begin{equation}
\label{equ:failure_induction}
    \begin{bmatrix}
        f(v) \\ F(v)
    \end{bmatrix}
    \!=\!
    \begin{cases}
        \!\begin{bmatrix}
            r_{\mysi} \\ \big[\chi_v\big]
        \end{bmatrix} 
        & {\normalfont \text{if }} \chi_v \in V, \\[4mm]
        \!\begin{bmatrix}
            h \big(f(u), c \big) \\
            F(u) \!+\!H \big(f(u),c \big)
        \end{bmatrix} \!\!\!
        & {\normalfont \text{otherwise.}} \\
    \end{cases}
\end{equation}
\end{lemma}
\begin{proof}
We just need to prove the second case in Equation~\ref{equ:failure_induction}. Let $w=\chi_u$, hence $\chi_v=\chi_uc=wc$. Since $wc \notin V$, we have $p_{wc}=p_w$ and $q_{wc}=q_wc$ (Lemma~\ref{thm:pwc}). 

If $p_{\chi_v}=p_{\chi_u}=\varepsilon$, $f(v)=f(u)=h(f(u),c)=\nulls$ and $F(v)=F(u)=H(f(u),c)=\myemptylist$, we have
\begin{align*}
\mybm{f(v)}{F(v)} & = \mybm{\nulls}{\myemptylist} =  \mybm{h(f(u),c)}{F(u)+H(f(u),c)}.
\end{align*}

Otherwise, since $p_{wc}=p_w$, $q_{wc}=q_wc$. Hence,
\begin{align*}
\mybm{f(v)}{F(v)} & = \mybm{g(q_{w}c)}{[p_{wc}]+G(q_{w}c)}  & \text{(Eq.~\ref{equ:failure})}\\
&= \mybm{h(f(u),c)}{p_w + G(q_w) + H(f(u),c)} & \text{(Eq.~\ref{equ:ground_uc_induction})}\\
&=  \mybm{h(f(u),c)}{F(u)+H(f(u),c)}& \text{(Eq.~\ref{equ:failure})}
\end{align*}

\end{proof}

\subsection{\lmmshort{} as a Finite-State Transducer (FST)}
\label{sec:appendix_fst}

In Section~\ref{sec:algo-discussions} we discussed that \lmmshort{} can be turned into a finite-state transducer (FST) by precomputing the transition function $\delta'(u,c)$ and the output function $\sigma'(u,c)$ to eliminate the failure transitions. 
As aforementioned in Section~\ref{sec:one-step-minpop-matching}, 
$\delta'(u,c)$ and $\sigma'(u,c)$ are essentially one-step MinPop matching:
\begin{align}
\mybm{\delta'(u,c)}{\sigma'(u,c)} & = \mybm{h(u,c)}{H(u,c)}  
\end{align}

If we precompute and store $\delta'(u,c)$ and $\sigma'(u,c)$, i.e.\ $h(u,c)$ and $H(u,c)$, Algorithm~\ref{alg:inference} can be rewritten as Algorithm~\ref{alg:inferencefst} (according to Lemma~\ref{thm:induction}), where the differences are lines~\ref{alg:inferencefst:step_start}-\ref{alg:inferencefst:loopend} in bold.

\begin{algorithm}[htb!]  
\DontPrintSemicolon
\mysetind
\SetKwProg{Fn}{Function}{:}{}
\nonl
\Fn{\textsc{\lmmshort{}}{\normalfont (}$w${\normalfont )}}{
  \text{tokens}, $u$, $i$ $\leftarrow$ \textsc{MatchLoop}$(w\mywhitespace{}, 0)$ \label{alg:inference2:call} \;
  \If{$i<|w|$ \mykw{or} $u \notin \{r, r_{\mysi}\}$ \label{alg:inference2:check_returns}}{tokens $\leftarrow [\myunk]$ \label{alg:inference2:reset}}
    \ElseIf{
    $
    u = r_{\mysi}  
      \mykw{ and } |\text{tokens}|=0 $ \label{alg:inference2:check_edge_case}}{
        tokens $\leftarrow \textsc{OriginalWordPiece}(\mysi)$ \label{alg:inference2:edge_case}
    }
  \Return tokens \label{alg:inference2:return}
}
\nonl
\SetNlSty{bfseries}{\color{black}}{}
\Fn{\textsc{MatchLoop}{\normalfont (}$s$, $i${\normalfont )}}{
  $u, \text{tokens}$ $\;\leftarrow\;$ $r, \myemptylist$ \;
  \While{$\mathbf{i < |s|}$ \mykw{and} $\mathbf{u \neq \nulls}$ \label{alg:inferencefst:step_start}}{
        \textbf{tokens} $\mathbf{\leftarrow}$ \textbf{\textsc{Extend}(tokens}, $\mathbf{H(u, s[i]))}$ \label{alg:inferencefst:append} \;
    $\mathbf{u \leftarrow h(u, s[i])}$ \label{alg:inferencefst:goto} \;
    $\mathbf{i \leftarrow i+1}$ \label{alg:inferencefst:loopend} \;
  }
  \Return tokens, $u$, $i$  \label{alg:inferencefst:final_return}
}
\caption{\lmmshort{} as an FST}
\label{alg:inferencefst}
\end{algorithm}

In Algorithm~\ref{alg:inferencefst}, we can see that the failure transitions are eliminated and \lmmshort{} works as an FST. The time complexity is trivially linear.

\end{document}